\documentclass{article}

\usepackage{microtype}
\usepackage{graphicx}
\usepackage{subcaption}
\usepackage{booktabs} 

\usepackage{hyperref}

\usepackage[preprint]{icml2026}

\usepackage[utf8]{inputenc}
\usepackage{textgreek}

\usepackage{amsmath}
\usepackage{amssymb}
\usepackage{mathtools}
\usepackage{amsthm}
\usepackage{colortbl}
\usepackage{multirow}
\usepackage{microtype}
\usepackage{subcaption}

\usepackage[capitalize,noabbrev]{cleveref}

\theoremstyle{plain}
\newtheorem{theorem}{Theorem}[section]

\newtheorem{lemma}[theorem]{Lemma}

\theoremstyle{definition}

\newtheorem{assumption}[theorem]{Assumption}
\theoremstyle{remark}
\newtheorem{remark}[theorem]{Remark}

\newcommand{\braces}[1]{\left\{#1\right\}}

\usepackage[textsize=tiny]{todonotes}

\usepackage{enumitem}

\icmltitlerunning{Many Needles in a Haystack: Active Hit Discovery for Perturbation Experiments}

\begin{document}

\twocolumn[
  \icmltitle{Many Needles in a Haystack: Active Hit Discovery for Perturbation Experiments}



  \icmlsetsymbol{equal}{*}

  \begin{icmlauthorlist}
    \icmlauthor{Andrea Rubbi}{equal,sanger,cambridgecs}
    \icmlauthor{Arpit Merchant}{equal,sanger}
    \icmlauthor{Samuel Ogden}{sanger}
    \icmlauthor{Amir Akbarnejad}{sanger,cambridgemed}
    \icmlauthor{{Pietro Liò}}{cambridgecs}
    \icmlauthor{Sattar Vakili}{mediatek}
    \icmlauthor{Mo Lotfollahi}{sanger,cambridgemed,cambridgestem}
  \end{icmlauthorlist}

  \icmlaffiliation{sanger}{Wellcome Sanger Institute, Wellcome Genome Campus, Hinxton, UK;}
  \icmlaffiliation{cambridgemed}{Cambridge Center for AI in Medicine, University of Cambridge, Cambridge, UK;}
  \icmlaffiliation{cambridgestem}{Cambridge Stem Cell Institute, University of Cambridge, Cambridge, UK;}
  \icmlaffiliation{cambridgecs}{Department of Computer Science and Technology, University of Cambridge, Cambridge, UK;}
  \icmlaffiliation{mediatek}{MediaTek Research, Cambridge, UK.}

  \icmlcorrespondingauthor{Andrea Rubbi}{ar36@sanger.ac.uk}
  \icmlcorrespondingauthor{Arpit Merchant}{am84@sanger.ac.uk}

  \icmlkeywords{Machine Learning, ICML}

  \vskip 0.3in
]



\printAffiliationsAndNotice{}  

\begin{abstract}
  High-throughput gene perturbation experiments can test several genetic interventions in parallel, yet experimental budgets remain limited. A central goal is hit discovery: identifying as many perturbations as possible whose phenotypic effect exceeds a predefined threshold. Pure exploration strategies are statistically inefficient, wasting budget on low-value regions. Bayesian optimization methods offer a principled alternative but target a single global optimum, over-exploiting dominant modes while neglecting other high-value regions. We formalize hit discovery as a sequential experimental design problem and propose Probability-of-Hit, an acquisition function that directly targets threshold exceedance by ranking candidates according to their posterior probability of being a hit. We prove asymptotic optimality of this approach and demonstrate strong empirical performance on both synthetic benchmarks and real biological immunology datasets, including up to 6.4\% improvement over baselines on the Schmidt IL-2 dataset.

\end{abstract}






\section{Introduction}\label{sec:introduction}
High-throughput gene perturbation technologies such as CRISPR-based screens and pooled single-cell assays have made it possible to probe the functional effects of genetic interventions~\citep{ben2018genome,maclean2018exploring}. 
Despite these advances, the space of candidate genes and gene combinations remains combinatorially vast~\citep{trapnell2015defining,worzfeld2017unique} and given their high monetary costs~\citep{dimasi2016innovation,dickson2009cost}, only a small fraction of all possible perturbations can be tested in practice. As a result, a central challenge in modern functional genomics is \emph{how to adaptively design perturbation experiments so as to efficiently discover biologically meaningful effects under limited experimental resources}.  In effect, adaptive perturbation design becomes a many-needle search problem, where sparse, high-impact perturbations must be identified within a combinatorially large space (haystack) using only a limited number of evaluations.

\begin{figure*}[ht!]
    \centering
    \includegraphics[width=1\linewidth]{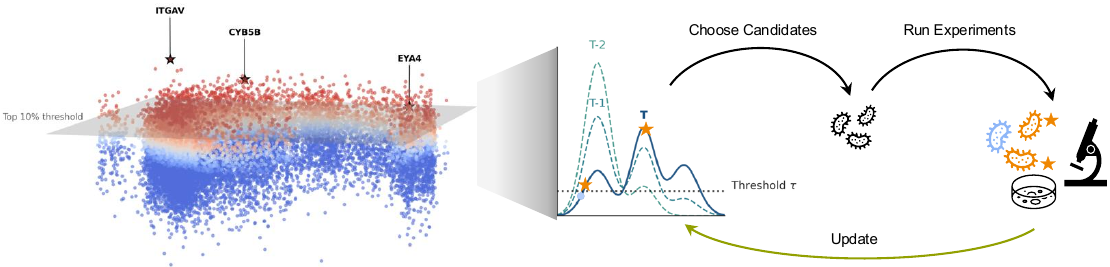}
    \caption{(a) 2D UMAP gene embeddings of Achilles features with the Z-axis denoting phenotypic response to sensitivity of leukemia cells to cytotoxic activity in human NK cells and colors indicating threshold levels. ITGAV, CYB5B and EYA4 are representative genes from three separate modes. (b) Overview of our closed-loop experimental design system. In each cycle, the algorithm samples a batch of genes from a probabilistic surrogate model. These are sent to the lab for validation through real-world (in-vivo or in-vitro) experiments. Finally, the surrogate is updated through the empirically generated responses.}
    \label{fig:hit_discovery_schematic}
\end{figure*}

To this end, a key objective in perturbation screening is \emph{hit discovery}~\citep{schneeberger2014using,hughes2011principles}: identifying genes (or gene combinations) whose perturbation induces a phenotypic response exceeding a predefined threshold of interest. This setting arises naturally in applications ranging from discovery of disease-relevant regulators~\citep{przybyla2022new}, to identification of drug targets~\citep{wang2023crispr}, and prioritization of candidate genes for downstream mechanistic studies~\citep{lim2022emerging}. Importantly, the goal is not necessarily to find a single optimal perturbation, but rather to recover as many distinct high-effect perturbations as possible within a fixed (small) budget~\citep{huang2024sequential,chappell2018single}.

This objective fundamentally differs from classical optimization and poses unique algorithmic challenges. Pure exploration strategies such as uncertainty sampling~\citep{lewis1994sequential} or maximum entropy selection~\citep{mackay1992information} prioritize sampling uncertain perturbations, making them statistically inefficient for hit discovery. Given an estimated 20,000 protein-coding genes~\citep{pertea2018chess} and small, sparse hit regions in the perturbation space, these methods expend substantial budget on low-value regions while remaining unlikely to yield hits. Conversely, standard Bayesian optimization methods such as Upper Confidence Bound~\citep[UCB,][]{srinivas2009gaussian} or Expected Improvement~\citep{jones1998efficient} are designed to locate a global optimum. While effective in unimodal settings, these approaches tend to over-exploit dominant modes and systematically neglect alternative high-response regions, leading to poor coverage in multimodal landscapes where biologically meaningful effects may be distributed across diverse and disconnected regions~\citep{kirsch2019batchbald}.

These limitations are particularly acute in biological systems, where response functions are often highly multimodal, structured, and heterogeneous~\citep{song2025decoding,vinasTorne2025systema}.
Gene perturbations may act through multiple independent pathways, exhibit context-specific effects, or interact nonlinearly with latent cellular states~\citep{gavriilidis2024perturbation}. In such settings, the optimal strategy for maximizing the number of discovered hits requires balancing exploitation of known high-response regions with targeted exploration of under-sampled but promising areas of the perturbation space (cf.\ Figure~\ref{fig:hit_discovery_schematic}), a balance that substantially differs from the exploration-exploitation trade-off in traditional Bayesian optimization.

In this work, we formalize hit discovery as a closed-loop experimental design problem and develop active learning algorithms explicitly tailored to this objective. Rather than optimizing for the maximum predicted response, we focus on acquisition strategies that maximize the expected recovery of threshold-exceeding perturbations under uncertainty. Our approach builds on Bayesian modeling of perturbation effects and leverages posterior predictive distributions to guide batch selection in a principled manner.

\paragraph{Contributions}
We make four key contributions:
\begin{itemize}
    \item We formally define the hit discovery problem, highlighting the shortcomings of existing exploration and optimisation-based strategies;
    \item We introduce a novel acquisition function Probability-of-Hit that directly targets the hit discovery objective;
    \item We theoretically prove the asymptotic optimality of our proposed algorithm for the task of recovery of threshold-exceeding perturbations under uncertainty;
    \item We empirically evaluate these methods on a suite of challenging synthetic benchmarks and single-cell perturbation datasets, demonstrating consistent improvements in hit recovery over standard baselines.
\end{itemize}

\section{Related Work}\label{sec:related_work}
\paragraph{Active learning and Bayesian optimization.}
Active learning and Bayesian optimization (BO) provide principled frameworks for sequential experiment design under uncertainty, typically by modeling an unknown response function with a probabilistic surrogate and selecting queries via an acquisition function \cite{shahriari2016review}. Classical BO methods such as Upper Confidence Bound (UCB), Expected Improvement (EI), and Probability of Improvement (PI) are primarily designed to identify a global optimum of the response function \cite{srinivas2010gpucb}. While these methods are statistically efficient for optimization, they are poorly aligned with objectives that require recovering \emph{multiple} high-response regions. In particular, UCB-style methods tend to concentrate samples around dominant modes, which is suboptimal in multimodal settings where the goal is to discover a diverse set of threshold-exceeding points rather than a single maximum \cite{srinivas2010gpucb}.

Pure exploration strategies, which prioritize uncertainty reduction independently of the predictive mean, have also been studied in active learning and bandit literature \cite{audibert2010best,bubeck2011pure}. However, such methods typically waste experimental budget on low-reward regions and achieve hit discovery rates comparable to random sampling when the proportion of hits is small \cite{jamieson2014lil}. These limitations motivate acquisition objectives that explicitly target hit recovery rather than uncertainty reduction or global optimization.


\paragraph{Thompson sampling and batch selection.}
Thompson sampling is a widely used Bayesian strategy that selects actions by sampling from the posterior distribution over functions and acting optimally under the sampled model \cite{russo2018tutorial}. It has attractive theoretical and empirical properties and naturally balances exploration and exploitation. Extensions to pure-exploration and fixed-confidence settings have been studied extensively in the bandit literature \cite{audibert2010best,bubeck2011pure}. While Thompson sampling is well suited for optimization and regret minimization, its behavior for hit discovery objectives has received comparatively less attention. In particular, naive Thompson sampling may still oversample dominant modes, motivating batch-aware variants that explicitly condition on virtual observations to encourage diversity and coverage \cite{russo2018tutorial}.

\paragraph{Active experimental design in biology.}
Active learning has increasingly been applied to biological experiment design, including protein engineering, molecular optimization, and genetic perturbation studies \cite{borkowski2020protein,wu2019directed}. Recent approaches leverage surrogate models over genes or perturbations, often using Gaussian processes or neural networks, to guide experiment selection \cite{zhang2023causal,li2025biobo}. Some methods explicitly optimize downstream biological objectives such as differential expression magnitude or fitness effects, while others aim to maximize information gain \cite{tosh2025combination,soper2025topk}. However, most existing approaches implicitly assume unimodal response landscapes or focus on optimizing average performance rather than maximizing the discovery of distinct high-effect perturbations. As a result, they may underperform in realistic biological settings characterized by heterogeneity, context dependence, and multiple mechanistic pathways \cite{hart2015highresolution,rottner2023calcoco2}.

\paragraph{DiscoBAX and objective-driven active learning.}
DiscoBAX introduces a general framework for \emph{objective-driven} active learning, in which the acquisition strategy is derived from a user-specified downstream scientific objective rather than generic uncertainty reduction or optimization criteria \cite{lyle2023discobax}. By casting experimental design as a Bayesian decision problem, DiscoBAX enables the construction of task-specific acquisition functions for objectives such as level-set estimation, rare-event discovery, or hypothesis testing. DiscoBAX primarily focuses on sequential (single-point) selection and generic objective templates, and does not explicitly address the challenges posed by multimodal response landscapes, large discrete action spaces, or batch experimental constraints typical of modern perturbation screens. Our work builds on the same decision-theoretic motivation but develops acquisition functions and batch-aware algorithms that are specifically tailored to scalable hit discovery in high-dimensional and multimodal biological settings \cite{hao2025perturboagent}.

\section{Setting}\label{sec:setting}
We study closed-loop experimental design for gene perturbation studies under a limited experimental budget. The experimenter sequentially selects batches of perturbations, observes phenotypic responses, and updates a predictive model to guide future selections. Our primary objective is \emph{hit discovery}: identifying as many perturbations as possible whose effect exceeds a predefined threshold.

\subsection{Perturbation Space}

Let $\mathcal{G}$ denote the set of candidate perturbations. In the simplest setting, $\mathcal{G}$ corresponds to a library of single-gene perturbations, typically comprising tens of thousands of genes. In combinatorial settings where perturbations correspond to pairs or higher-order combinations of genes, the size of $\mathcal{G}$ grows rapidly, often exceeding hundreds of millions of candidates.

Each perturbation $g \in \mathcal{G}$ induces an unknown scalar phenotypic response
\[
f(g) \in \mathbb{R},
\]
which may represent, for example, a transcriptional score, fitness effect, or pathway activation metric. We assume that querying $f(g)$ is expensive and corresponds to performing a biological experiment.

\subsection{Hit Discovery Objective}

Given a user-specified threshold $\tau \in \mathbb{R}$, we define the (unknown) set of hits as
\[
\mathcal{G}^\star := \{ g \in \mathcal{G} \mid f(g) > \tau \}.
\]
The experiment proceeds over $T$ rounds. At each round $t = 1, \dots, T$, the experimenter selects a batch $B_t \subset \mathcal{G}$ of fixed size $b$ and observes noisy evaluations $\{ y_g : g \in B_t \}$, where
\[
y_g = f(g) + \varepsilon_g,
\]
and $\varepsilon_g$ denotes observation noise.
Let
\[
B_t^+ := \{ g \in B_t \mid f(g) > \tau \}
\]
denote the hits discovered at round $t$, and define the cumulative discovered hit set as
\[
\mathcal{G}_T := \bigcup_{t=1}^T B_t^+.
\]
Our goal is to maximize hit recovery, measured either as the absolute number of discovered hits $|\mathcal{G}_T|$ or as the hit ratio
\[
\mathrm{HitRatio}_T := \frac{|\mathcal{G}_T|}{|\mathcal{G}^\star|}.
\]

\subsection{Sequential and Batch Protocol}

We consider a batch setting motivated by modern high-throughput experiments, where multiple perturbations can be executed in parallel. Importantly, selections within a batch must be made without access to the true outcomes of other perturbations in the same batch. 

Let $n_t = tb$ denote the total number of perturbations evaluated up to round $t$. At each round, the experimenter conditions on the dataset
\[
\mathcal{D}_t := \{ (g, y_g) : g \in B_s,\, s < t \}
\]
to guide the selection of $B_t$.

\subsection{Predictive Model}

We assume access to a probabilistic predictive model that, conditioned on $\mathcal{D}_t$, provides a posterior predictive distribution for each $g \in \mathcal{G}$, characterized by a mean $\mu_t(g)$ and uncertainty $\sigma_t(g)$. Our algorithms are agnostic to the specific choice of model, and this abstraction encompasses Gaussian processes~\citep{rasmussen2006gaussian}, Bayesian neural networks~\citep{neal1996bayesian, blundell2015weight}, and neural surrogates equipped with calibrated uncertainty estimates~\citep{gal2016dropout, lakshminarayanan2017simple}.

Crucially, the predictive model induces a posterior distribution over the unknown response function $f$. This enables reasoning about counterfactual observations and supports acquisition strategies based on expected utility or posterior sampling.

\subsection{Challenges}

This setting presents several challenges that distinguish hit discovery from classical optimization. The hit discovery objective is fundamentally different from regret minimization in traditional Bayesian optimization: neither implies the other, leading to a fundamentally different exploration-exploitation trade-off that requires novel algorithmic developments and analyses distinct from those used in standard bandit and optimization settings. Batch selection introduces additional complexity, as points must be chosen without observing the outcomes of other points in the same batch. Furthermore, the action space is large and often discrete, limiting applicability of gradient-based acquisition optimization. These considerations motivate acquisition functions and batch construction strategies that are explicitly aligned with the hit discovery objective. We show that a straightforward algorithm achieves near-optimal performance.

\section{An Algorithm to Discover Many Needles in a Haystack}\label{sec:algorithms}
We present an algorithm for hit discovery under the setting of Section~\ref{sec:setting}. We assume access to a probabilistic surrogate model that provides, for each candidate perturbation $g \in \mathcal{G}$ at round $t$, a predictive mean $\mu_t(g)$ and predictive variance $\sigma_t^2(g)$. We focus on batch selection with the constraint that each perturbation can be evaluated at most once.

\subsection{Probability-of-Hit}
\label{sec:probhit}

A natural baseline is to select the top-$k$ candidates by predicted mean $\mu_t(g)$. However, this approach ignores uncertainty and can be brittle when the surrogate is miscalibrated or has limited coverage. A principled alternative is to rank candidates by their posterior probability of exceeding the hit threshold $\tau$. Define the posterior hit probability
\begin{equation}
p_t(g) \;:=\; \mathbb{P}\bigl(f(g) > \tau \mid \mathcal{D}_t \bigr).
\label{eq:pt_def}
\end{equation}
When the predictive distribution is Gaussian,
$f(g) \mid \mathcal{D}_t \sim \mathcal{N}(\mu_t(g), \sigma_t^2(g))$,
this probability has the closed form
\begin{equation}
p_t(g) \;=\; 1 - \Phi\!\left(\frac{\tau - \mu_t(g)}{\sigma_t(g)}\right),
\label{eq:pt_gauss}
\end{equation}
where $\Phi$ denotes the standard normal CDF is the observation-noise variance.

At each round, we select $b$ unsampled candidates with the largest $p_t(g)$. This strategy balances mean and uncertainty in a way that is directly aligned with hit discovery.

\begin{algorithm}[t]
\caption{Probability-of-Hit}
\label{alg:poh}
\begin{algorithmic}[1]
\REQUIRE Candidate set $\mathcal{G}$, threshold $\tau$, batch size $b$, rounds $T$, surrogate posterior providing $(\mu_t, \sigma_t^2)$
\STATE Initialize $\mathcal{D}_1 \gets \emptyset$
\FOR{$t = 1$ to $T$}
    \STATE Fit/condition surrogate on $\mathcal{D}_t$
    \STATE $C_t \gets \mathcal{G} \setminus \bigcup_{s<t} B_s$
    \STATE Compute $p_t(g)$ for all $g \in C_t$ using~\eqref{eq:pt_def}
    \STATE $B_t \gets$ top-$b$ elements of $C_t$ by $p_t(g)$ (ties broken randomly)
    \STATE Evaluate $\{y_g : g \in B_t\}$
    \STATE $\mathcal{D}_{t+1} \gets \mathcal{D}_t \cup \{(g, y_g) : g \in B_t\}$
\ENDFOR
\end{algorithmic}
\end{algorithm}




\subsection{Analysis}
\label{sec:analysis}

We analyze the performance of Probability-of-Hit (PoH) introduced in Section~\ref{sec:probhit} under the experimental design setting of Section~\ref{sec:setting}. Our analysis provides high-probability guarantees for hit discovery under standard assumptions on the surrogate model and the underlying response function.

\paragraph{Probabilistic model.}
We assume that the unknown response function $f:\mathcal{G}\to\mathbb{R}$ is modeled using a probabilistic surrogate that, conditioned on the history $\mathcal{D}_t$, induces a posterior predictive distribution for each candidate $g\in\mathcal{G}$ with mean $\mu_t(g)$ and standard deviation $\sigma_t(g)$. While the analysis applies to any calibrated Bayesian predictive model, for clarity we consider Gaussian process (GP) posteriors, for which such uncertainty quantification is standard. Our analysis relies only on the following two assumptions, which are standard and hold for GPs but apply more broadly to any well-calibrated probabilistic surrogate.

\begin{assumption}[Posterior concentration]
\label{ass:gp_conf}
For any confidence level $\delta\in(0,1)$, there exists a nondecreasing sequence $\{\beta_t\}_{t\ge 1}$ such that, with probability at least $1-\delta$, the surrogate posterior satisfies
\[
|f(g) - \mu_t(g)| \le \beta_t\,\sigma_t(g),
\quad \forall g \in \mathcal{G},\ \forall t.
\]
\end{assumption}
Here $\beta_t$ is a confidence interval width multiplier used to calibrate uncertainties. Such uniform concentration bounds are well known for Gaussian processes and related Bayesian models, with $\beta_t$ typically growing as $\mathcal{O}\bigl(\sqrt{\log(t/\delta)}\bigr)$~\citep[see, e.g.,][]{srinivas2009gaussian}.

\begin{assumption}[Margin condition]
\label{ass:margin}
There exists a margin parameter $\varepsilon > 0$ such that the set
\[
\mathcal{G}^\star_\varepsilon := \bigl\{ g \in \mathcal{G} : f(g) \ge \tau + \varepsilon \bigr\}
\]
is nonempty.
\end{assumption}
This mild condition excludes degenerate cases in which all candidate perturbations lie arbitrarily close to the hit threshold, and is standard in thresholded discovery and level-set estimation problems.

\paragraph{Performance metric.}
Our primary object of interest is the number of hits discovered by PoH up to round $T$,
\[
H_T := \sum_{t=1}^T \sum_{g \in B_t} \mathbf{1}\{ f(g) > \tau \},
\]
where $B_t$ denotes the batch selected at round $t$.
We analyze $H_T$ in a high-probability sense, characterizing its dependence on the experimental budget $T$ and the batch size $b$.

\paragraph{Information gain.}
Let $k$ denote the kernel of the surrogate model and let $\lambda>0$ be the observation-noise variance.
For any set $A \subseteq \mathcal{G}$, define the mutual information between the noisy observations
$y_A := \{f(g)+\varepsilon_g : g\in A\}$ and the latent function values $f_A := \{f(g): g\in A\}$ as
\[
I(y_A ; f_A) \;=\; \frac{1}{2}\log\det\!\left(I + \lambda^{-1} K_A\right),
\]
where $K_A$ is the kernel matrix indexed by $A$.
The \emph{maximum information gain} after $n$ observations is defined as
\[
\gamma_n \;:=\; \max_{A \subseteq \mathcal{G},\, |A| = n} I(y_A ; f_A).
\]

The following theorem states our main performance guarantee for PoH under
Assumptions~\ref{ass:gp_conf} and~\ref{ass:margin}.

\begin{theorem}[Hit discovery guarantee for Probability-of-Hit]
\label{thm:poh_hits}
Consider the experimental design setting for hit discovery described in Section~\ref{sec:setting}, and the Probability-of-Hit (PoH) algorithm given in Algorithm~\ref{alg:poh}.
Under Assumptions~\ref{ass:gp_conf} and~\ref{ass:margin}, with probability at least $1-\delta$, the number $H_T$ of hits discovered by PoH up to round $T$ satisfies
\begin{align}\nonumber
&H_T
\;\ge\;
\frac{Tb}{2}
\;-\;
\sqrt{\tfrac{1}{2}Tb \log\tfrac{2}{\delta}}
\\[-1mm]\label{eq:thm_poh_batch_final}
&-\;
\mathcal{O}\!\left(
\frac{\beta_{Tb}^{3}}{\epsilon^2}\,
\sqrt{
\frac{\gamma_{Tb}}{\log(1+\lambda^{-1})}
\left(
Tb \;+\; \frac{b^2\,\gamma_T}{\log(1+\lambda^{-1})}
\right)
}
\right).
\end{align}
\end{theorem}

\begin{remark}[Asymptotic implications]
For many commonly used kernels, the maximum information gain grows sublinearly.
In particular, for the squared exponential (SE) kernel in $d$ dimensions,
$\gamma_n = \mathcal{O}((\log n)^{d+1})$, and for Mat\'ern kernels with smoothness parameter
$\nu > 1$, $\gamma_n = \tilde{\mathcal{O}}(n^{\frac{d}{2\nu + d}})$.
Such bounds are well established in the GP bandits literature
\citep{vakili2021informationgain}.

Consequently, for any kernel such that $\gamma_n = o(n)$, PoH discovers at least
$\frac{Tb}{2} - o(Tb)$ hits with high probability.
\end{remark}

\paragraph{Proof sketch.}
The proof relates realized hit count $H_T$ to cumulative posterior hit probabilities via a martingale concentration argument over the $Tb$ queried perturbations.
Next, using the PoH selection rule and Assumption~\ref{ass:gp_conf}, we show that whenever a margin-hit in $\mathcal{G}^\star_\epsilon$ remains unqueried, selecting a true non-hit forces the queried perturbation to have sufficiently large posterior uncertainty, with the threshold scaling as $\sigma_t(\cdot)\gtrsim \epsilon/\beta_t$.
Finally, we control the resulting cumulative uncertainty cost under delayed (batch) feedback by bounding $\sum_{t=1}^T\sum_{g\in B_t}\sigma_t(g)$ using a batched information-gain argument (Lemma~\ref{lem:batch_sigma_sum}).
Combining these ingredients yields~\eqref{eq:thm_poh_batch_final}.
A complete proof is provided in Appendix~\ref{app:proofs}.

\begin{table*}[h!]
\caption{Dataset complexity statistics with arrows indicating whether larger ($\uparrow$) or smaller ($\downarrow$) values correspond to increased complexity.
Local Smoothness ($\mathcal{S}$) is the average absolute difference in target values between neighboring points in feature space.
Hit Cluster Count ($n_{\text{clusters}}$) and Hit Spread ($\sigma_{\text{hits}}$) denote the number of spatially distinct clusters of hit points and standard deviation of pairwise distances between them. 
Feature-Response Correlation ($\rho_{d,y}$) captures correlation between pairwise distances and differences in target values.}
\label{tab:datasets_summary}
\centering
\setlength{\tabcolsep}{3pt}
\begin{subtable}[t]{0.49\textwidth}
\centering
\caption{Synthetic Datasets}
\label{tab:datasets_synth}
\small
\begin{tabular}{lcccc}
\toprule
Dataset
& $\mathcal{S}\!\downarrow$
& $n_{\mathrm{clusters}}\!\uparrow$
& $\sigma_{\mathrm{hits}}\!\uparrow$
& $\rho_{d,y}\!\uparrow$ \\
\midrule
Sine (1D)                   & 0.062 & 1.0 & 0.171 & 0.661 \\
Sine (2D)                   & 0.108 & 2.0 & 1.476 & 0.001 \\
Branin-Hoo (2D)             & 0.057 & 2.5 & 1.811 & 0.257 \\
Perturbation-Pathways (4D)  & 0.419 & 2.4 & 2.522 & 0.056 \\
Perturbation-SEM (6D)       & 0.750 & 1.2 & 2.871 & 0.054 \\
\bottomrule
\end{tabular}
\end{subtable}
\begin{subtable}[t]{0.49\textwidth}
\centering
\caption{Real Datasets}
\label{tab:datasets_real}
\small
\begin{tabular}{lcccc}
\toprule
Dataset
& $\mathcal{S}\!\downarrow$
& $n_{\mathrm{clusters}}\!\uparrow$
& $\sigma_{\mathrm{hits}}\!\uparrow$
& $\rho_{d,y}\!\uparrow$ \\
\midrule
Schmidt IFN-$\gamma$ & 0.22 & 12.7 & 0.044 & 0.010 \\
Schmidt IL-2         & 0.21 & 17.3 & 0.104 & 0.048 \\
Sanchez Tau          & 0.67 & 13.0 & 0.060 & 0.031 \\
Zhu SARS-CoV-2       & 1.00 & 9.7  & 0.134 & 0.059 \\
Zhuang NK            & 1.76 & 13.0 & 0.042 & 0.006 \\
\bottomrule
\end{tabular}
\end{subtable}
\vskip -0.1in
\end{table*}
\section{Experiments}\label{sec:experiments}
In this section, we describe our extensive empirical evaluation of the experiment design protocol with a view to answering the following questions: \underline{\textit{RQ1}}: How does dataset complexity impact method performance and experiment parameters? \underline{\textit{RQ2}}: Can hit-aware acquisition functions improve discovery of high-value targets compared to state-of-the-art baselines? \underline{\textit{RQ3}}: Do the proposed methods generalize across diverse biological assays?

\subsection{Datasets}\label{sec:experiments:datasets}

\paragraph{Synthetic Datasets.}
We employ three benchmark synthetic datasets namely (a) \textbf{Sine (1D)}: one-dimensional sinusoidal function with additive Gaussian noise and a single global maximum, (b) \textbf{Sine (2D)}: weighted combination of two sinusoidal components with learnable structure, and (c) \textbf{Branin-Hoo (2D)}\footnote{\url{https://www.sfu.ca/~ssurjano/branin.html}}: classical Bayesian optimization dataset inverted to yield three global maxima (originally minima).
Additionally, we design two novel synthetic datasets that more closely reflect molecular biology.
\begin{itemize}[leftmargin=*,nosep]
\item \textbf{Perturbation-Pathways (4D)}: simulated pathway where genes with related functions are clustered into modules. 
The key insight is that perturbing \emph{any} gene in a given pathway produces a strong phenotypic response while Gaussian activation creates learnable spatial structures.

\item \textbf{Perturbation-SEM (6D)}: simulated single-cell perturbation with explicit causal structure. 
capturing phenomena where similar perturbations can have opposite effects under different cellular and environmental contexts.
\end{itemize}

\paragraph{Real-World Datasets.} 
We evaluate our protocol on five immunology datasets from the GeneDisco benchmark ~\citep{mehrjou2021genedisco}.
Schmidt-IFN$\gamma$ and Schmidt-IL2~\citep{schmidt2021crispr} measure changes in read-counts of Interferon-$\gamma$ and Interleukin-2 cytokine secretions in T cells. 
Sanchez-Tau~\citep{sanchez2021genome} screens for modulators of tau phosphorylation in human iPSC-derived neurons. 
Zhu-SARS-CoV-2~\citep{zhu2021genome} identifies host factors required for SARS-CoV-2 infection.
Zhuang-NK~\citep{zhuang2019genome} screens genes regulating sensitivity of leukemia cells to cytotoxic activity of natural killer (NK) cells.
Each dataset contains $\sim$18,000 genes with continuous phenotypic readouts. We represent genes using Achilles~\citep{dempster2019extracting}, a set of 808-dimensional feature vectors.

Table~\ref{tab:datasets_summary} summarizes dataset complexity scores.
We include complete details required for constructing the synthetic datasets and complexity metrics in Appendix~\ref{app:datasets}.

\begin{figure*}[htbp]
    \centering
    \includegraphics[width=\textwidth]{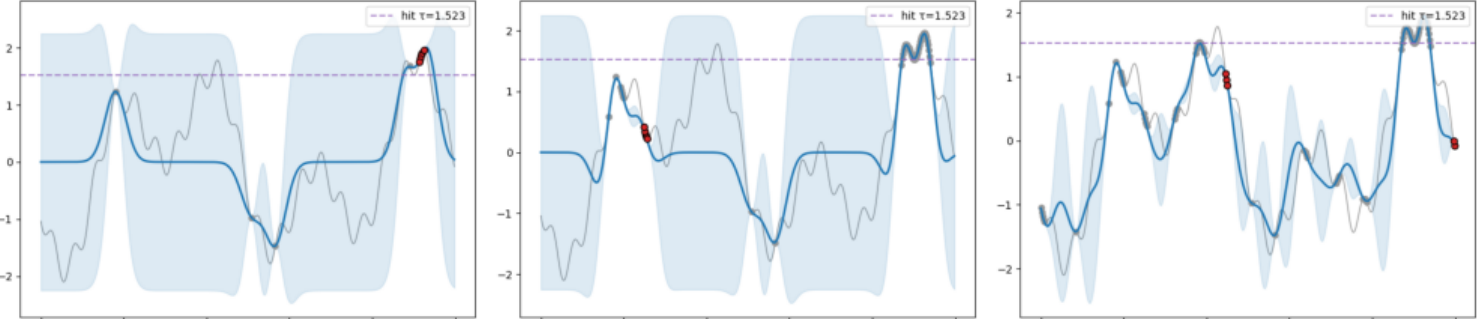}
    \caption{Illustration of the learning process underlying Probability-of-Hit. The solid grey line represents the ground-truth distribution, the dotted grey line represents $\tau=10\%$. The solid blue line and the shaded blue region around it represent the current surrogate and the model’s uncertainty with the red points denoting the samples chosen in each cycle.}
    \label{fig:sine_probability_hit}
\end{figure*}

\begin{table*}[ht]
\footnotesize
\centering
\setlength\tabcolsep{2pt}
\caption{Performance of various acquisition functions on synthetic datasets measured via mean and standard deviation of cumulative hit ratio at $\tau=10\%$ for $b=5$ after the 5th and 10th cycles, averaged across 20 seeds. Values highlighted in blue indicate best performance.}
\label{tab:main_synthetic_results}
\begin{tabular}{lcccccccc}
\toprule
\multirow{2}[2]{*}{\textbf{Acquisition}} & \multicolumn{2}{c}{\textbf{Sine (2D)}} & \multicolumn{2}{c}{\textbf{Branin-Hoo (2D)}} & \multicolumn{2}{c}{\textbf{Perturbation-Pathways (4D)}} & \multicolumn{2}{c}{\textbf{Perturbation-SEM (6D)}} \\
\cmidrule(lr){2-3} \cmidrule(lr){4-5} \cmidrule(lr){6-7} \cmidrule(lr){8-9}
& $t = 5$ & $t = 10$ & $t = 5$ & $t = 10$ & $t = 5$ & $t = 10$ & $t = 5$ & $t = 10$ \\
\midrule
Random & $0.1 \pm 0.04$ & $0.14 \pm 0.04$ & $0.11 \pm 0.04$ & $0.15 \pm 0.04$ & $0.09 \pm 0.04$ & $0.13 \pm 0.05$ & $0.12 \pm 0.04$ & $0.16 \pm 0.05$ \\
Top-K Greedy & \cellcolor{blue!15} $0.57 \pm 0.06$ & $0.84 \pm 0.15$ & $0.34 \pm 0.06$ & $0.49 \pm 0.06$ & $0.42 \pm 0.13$ & $0.58 \pm 0.16$ & \cellcolor{blue!15} $0.32 \pm 0.09$ & $0.46 \pm 0.13$ \\
Thompson Sampling & $0.54 \pm 0.06$ & $0.83 \pm 0.05$ & $0.28 \pm 0.06$ & $0.45 \pm 0.08$ & $0.37 \pm 0.08$ & $0.58 \pm 0.09$ & $0.18 \pm 0.05$ & $0.3 \pm 0.07$ \\
Thompson-Hit & $0.46 \pm 0.07$ & $0.76 \pm 0.05$ & $0.23 \pm 0.05$ & $0.37 \pm 0.07$ & $0.29 \pm 0.07$ & $0.49 \pm 0.09$ & $0.15 \pm 0.04$ & $0.23 \pm 0.08$ \\
Probability-of-Hit & \cellcolor{blue!15} $0.57 \pm 0.06$ & \cellcolor{blue!15} $0.88 \pm 0.05$ & \cellcolor{blue!15} $0.35 \pm 0.06$ & \cellcolor{blue!15} $0.5 \pm 0.07$ & \cellcolor{blue!15} $0.44 \pm 0.13$ & \cellcolor{blue!15} $0.62 \pm 0.14$ & \cellcolor{blue!15} $0.32 \pm 0.09$ & \cellcolor{blue!15} $0.49 \pm 0.12$ \\
\bottomrule
\end{tabular}
\end{table*}

\subsection{Experimental Setup}

\paragraph{Acquisition Functions.}
We compare Probability-of-Hit against state-of-the-art acquisition functions across four categories from the DiscoBAX benchmark~\citep{lyle2023discobax}:
\begin{itemize}[leftmargin=*,nosep]
    \item \textbf{Random}: Naive baseline that uniformly samples candidates without replacement.
    \item \textbf{Top-K Greedy}~\cite{neiswanger2021bayesian}: Active learning algorithm that selects candidates with highest predicted mean $\mu(x)$.
    \item \textbf{Thompson Sampling}~\citep{chowdhury2017kernelized}: Bandit algorithm that samples from a learned posterior distribution and selects candidates with highest values.
    \item \textbf{DiscoBAX}~\citep{lyle2023discobax}: Bayesian algorithm execution that selects subsets to maximize information gain.\footnote{We report results for DiscoBAX only on a small synthetic experiment due to large computational requirements and relatively poor performance (Appendix \ref{app:discobax}).}
\end{itemize}
Additionally, we implement \textbf{Thompson-Hit} (cf. Appendix~\ref{app:thompson_hit} for the pseudocode) as a natural variant of Thompson Sampling modified for hit-discovery.

\paragraph{Model.}
We pair all acquisition functions with a Gaussian Process model with RBF kernel for synthetic datasets and a Bayesian MLP model equipped with Monte Carlo Dropout~\citep{gal2016dropout} for uncertainty quantification for real datasets. The latter consists of two fully-connected hidden layers of dimensions 128 and 64 each with ReLU activations and dropout probability $p=0.2$. We use 50 Monte Carlo samples for posterior approximation and train for 100 epochs with early stopping patience 10.

\paragraph{Active Learning Protocol.} 
We vary batch size $b \in \braces{2, 5, 10}$ for synthetic data, and $b = 200$ for real data. 
We choose thresholds $\tau \in \braces{5\%, 10\%, 20\%}$.
Each experiment is conducted over $T=10$ cycles.
Within each cycle $t \in T$, first, the model is trained on the available labeled pool of genes. 
Next, the acquisition function selects $b$ candidate genes from the unlabeled pool based on the updated model's predictions. 
Finally, true labels are revealed for the candidates and the training set is updated.

\paragraph{Metrics.}
We define two metrics for evaluating acquisition function and model performance:

\begin{itemize}[leftmargin=*,nosep]
\item \textbf{Threshold Cumulative Hit Rate} (CHR@$\tau$): true hits recovered after $T$ cycles where hits are candidate genes in the top $\tau$ fraction by ground-truth value.
\item \textbf{Symmetric mean absolute percentage error} (SMAPE): relative error between predicted and actual values with under and over estimates treated equally.
\end{itemize}
CHR@$\tau$ captures the goal of maximizing discoveries under a fixed experimental budget while SMAPE measures the degree of accurate exploration of the perturbation space.






\begin{table*}[htbp]
\centering
\caption{CHR@10 for different acquisition functions on the real-world datasets. Experiments are conducted across 10 cycles of $b=200$ and averaged across 5 seeds. Values in \textbf{bold} indicate the best-performing method per row (dataset).}
\label{tab:full_results_10cycles_best_features}
\footnotesize
\resizebox{\textwidth}{!}{%
\begin{tabular}{llccccc}
\toprule
\textbf{Dataset} & \textbf{Probability-of-Hit} & \textbf{Thompson} & \textbf{Thompson-Hit} & \textbf{Top-K Greedy} & \textbf{Random} \\
\midrule
Schmidt IFN-$\gamma$ & \textbf{337.2$\pm$48} & 326.0$\pm$33 & 308.0$\pm$64 & 315.6$\pm$55 & 213.0$\pm$10 \\
Schmidt IL-2 & \textbf{575.4$\pm$17} & 521.4$\pm$64 & 538.2$\pm$36 & 516.6$\pm$60 & 221.2$\pm$9 \\
Sanchez Tau  & \textbf{389.4$\pm$11} & 377.4$\pm$13 & 375.2$\pm$16 & 376.0$\pm$13 & 213.0$\pm$10 \\
Zhu SARS-CoV-2 & 251.6$\pm$9 & 261.0$\pm$9 & \textbf{261.2$\pm$6} & 252.8$\pm$28 & 219.4$\pm$11 \\
Zhuang NK   &  246.0$\pm$28 & 251.6$\pm$19 & \textbf{252.4$\pm$19} & 245.2$\pm$15 & 215.4$\pm$23 \\
\bottomrule
\end{tabular}
}
\end{table*}

\subsection{Results}
\label{sec:results_on_synthetic}


Figure~\ref{fig:sine_probability_hit} illustrates the learning process for Probability-of-Hit. With increasing cycles, it converges to the ground-truth distribution with high confidence and locates hits.


\paragraph{Probability-of-Hit consistently outperforms baselines.}
Table~\ref{tab:main_synthetic_results} reports CHR@10 for $b=5$ after 5th and 10th cycles for all baselines. 
Across all datasets, cycles, and 20 random seeds, Probability-of-Hit significantly outperforms all baselines (paired Wilcoxon tests, $p \le 2.3\times10^{-4}$), with large effect sizes relative to Random ($\delta=0.97$), Thompson-Hit ($\delta=0.88$), Thompson ($\delta=0.76$), and moderate but consistent gains over Top-K Greedy ($\delta=0.41$).

\begin{figure}[h]
    \centering
    \includegraphics[width=0.9\columnwidth]{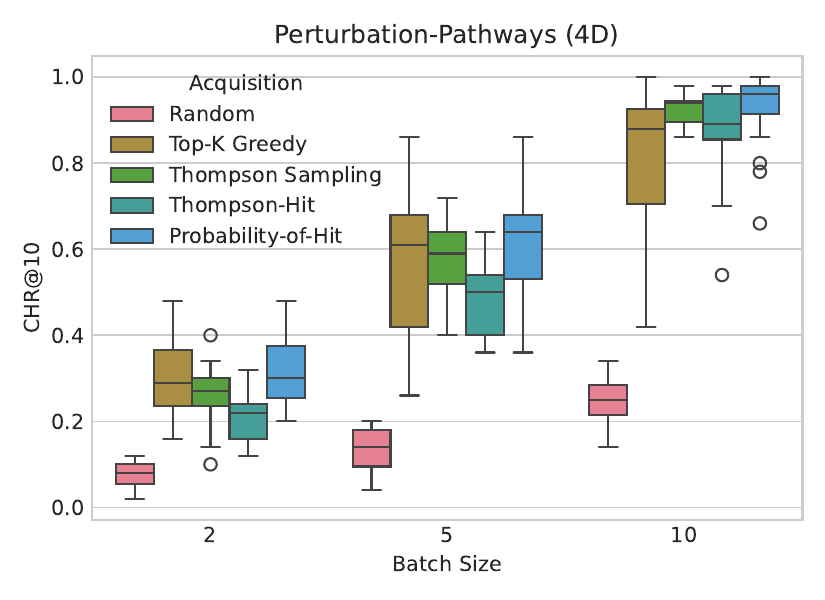}
    \caption{Impact of batch size ($b$) on CHR@10 for Perturbation-Pathways (4D). Probability-of-Hit outperforms all baselines.}
    \label{fig:pert_path_batch_size}
\end{figure}


\paragraph{Sensitivity analysis.}
Figure~\ref{fig:pert_path_batch_size} shows a 4-fold improvement from $b=2$ to $b=10$ on the Perturbation-Pathways (4D) dataset.
In contrast, the effect of finding between top $\tau = 5\%$ and $\tau = 20\%$ performers is limited once the model has sufficient training data.
This has important practical implications: increasing the number of experiments per cycle provides substantially more value than refining hit threshold.

\begin{figure}[h]
    \centering
    \includegraphics[width=0.85\columnwidth]{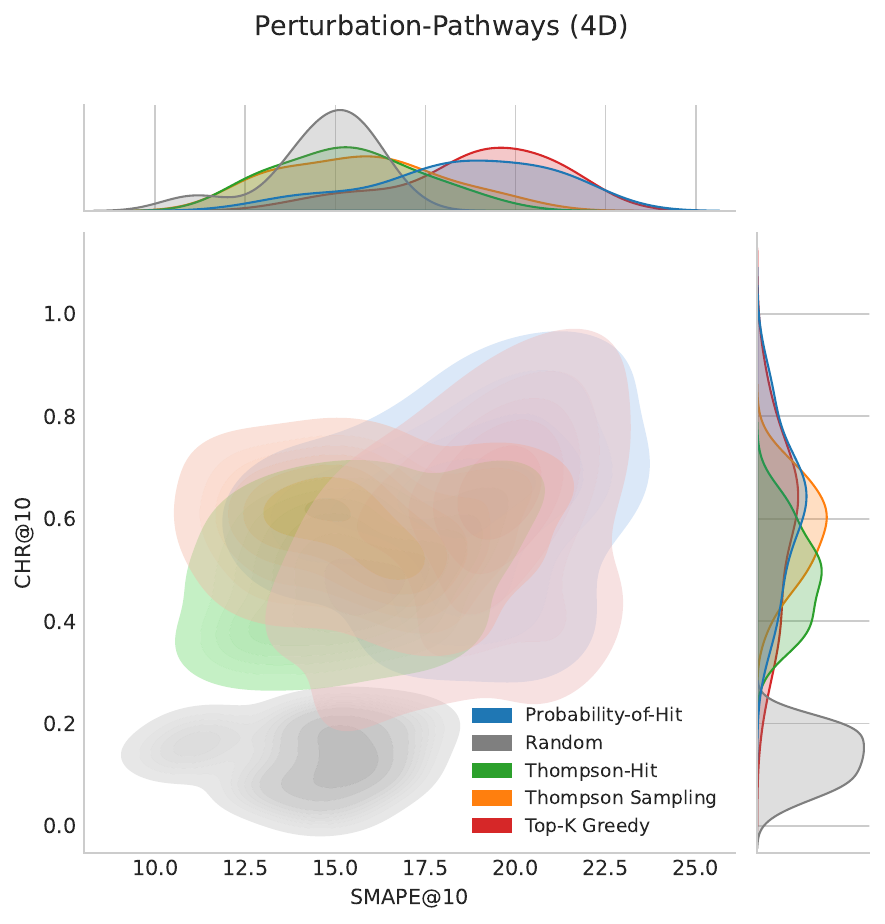}
    \caption{Trade-off at $b=5$ between CHR@10 and SMAPE@10 on Perturbation-Pathways (4D) across acquisition functions.}
    \label{fig:pert_path_tradeoff}
\end{figure}

\paragraph{Exploration-Exploitation Trade-off.}
Figure~\ref{fig:pert_path_tradeoff} reveals a trade-off: methods that achieve high hit recovery (Probability-of-Hit, Top-K Greedy) exhibit higher SMAPE, while Random achieves the lowest SMAPE, but poor hit recovery. This reflects the fundamental tension between building an accurate global surrogate model (\emph{exploration}) and focusing on promising regions (\emph{exploitation}). Thompson-Hit offers an alternative balance point with lower SMAPE at the cost of reduced hit recovery.

\paragraph{Hit-Recovery based methods outperform baselines on real datasets.}
Table~\ref{tab:full_results_10cycles_best_features} reports total hits recovered after 10 cycles for $b=200$ and $\tau=10\%$ across acquisition functions. Results are averaged over 5 seeds. 
Probability-of-Hit significantly outperforms Random (Cliff' $\delta = 0.894$, paired $t$-test $p < 0.0001$, Wilcoxon signed-rank $p < 0.0001$), recovering, on average, 107 more hits after 10 cycles. 

Probability-of-Hit displays the largest advantage on Schmidt IL-2. This dataset has the lowest local smoothness ($\mathcal{S} = 0.21$) in our complexity analysis, suggesting Probability-Hit excels when the phenotype landscape is learnable.
However, Thompson-Hit displays best performance on Zhuang NK which has the weakest feature-phenotype correlation and more fragmented hit structures indicating that increased exploration is crucial. 

Due to space constraints, we defer additional results on sensitivity analysis, tradeoffs, and hit recovery across cycles to Appendix~\ref{app:additional_results}.

\section{Conclusion}\label{sec:conclusions}

We study \emph{hit discovery} as a principled objective for closed-loop perturbation experiments, motivated by the need to identify high-effect perturbations under a stringent experimental budget. We show that commonly used exploration- and optimisation-driven acquisition strategies are misaligned with this goal, particularly in multimodal and heterogeneous response landscapes typical of biological systems.

To address this gap, we propose \emph{Probability-of-Hit} to directly target threshold exceedance by ranking candidates according to their posterior probability of being a hit. We provide theoretical guarantees that establish asymptotic optimality under standard uncertainty calibration assumptions and demonstrate consistent empirical gains across a diverse suite of synthetic benchmarks and large-scale CRISPR screening datasets. Our results show that explicitly aligning acquisition with the hit discovery objective leads to improved recovery of biologically meaningful perturbations without requiring additional model complexity. By reframing perturbation screening as a many-needle search problem rather than a single-optimum optimisation task, we provide a foundation for more efficient and biologically relevant experimental strategies in functional genomics and beyond. 

\section{Limitations and Future Work}\label{sec:limitations}

This work has some limitations that suggest promising directions for future research. Our analysis assumes access to a well-calibrated probabilistic surrogate model. Although Probability-of-Hit is model-agnostic, its performance depends on the quality of posterior uncertainty estimates. Improving calibration for deep- or large-scale surrogates, particularly in sparse-data regimes, remains an important challenge.
While we consider batch selection, we do not explicitly optimise diversity within a batch beyond what is induced by posterior uncertainty. Incorporating explicit diversity constraints, structured batch construction, or combinatorial interactions between perturbations may further improve performance in high-throughput experimental settings.
Finally, we primarily study single-perturbation screens. Extending hit-aware acquisition to combinatorial perturbations, conditional contexts (e.g., cell types or environmental states), and multi-objective biological readouts represents an exciting avenue for future work.


\section*{Impact Statement}\label{sec:impact}

In this paper, we present Probability-Hit, a sampling method that improves the efficiency of experimental design in high-throughput perturbation studies. By enabling faster and more reliable discovery of high-effect perturbations, our approach has the potential to reduce experimental cost and accelerate biological discovery in areas such as functional genomics, drug target discovery, and cellular engineering.

This method is a general-purpose algorithmic tool and does not raise immediate ethical concerns beyond those already present in large-scale biological experimentation. As with any technology that accelerates biological discovery, downstream applications should be governed by appropriate experimental oversight, data governance, and ethical review. We do not anticipate negative social impacts that could arise directly from this work.


\bibliography{references}
\bibliographystyle{icml2026}


\newpage
\appendix
\onecolumn

\section{Additional Related Work}\label{app:related_work}


\paragraph{Level-set estimation and thresholded objectives.}
Hit discovery is closely related to level-set estimation and contour-finding problems, where the goal is to identify the subset of inputs whose function value exceeds a given threshold. Prior work has proposed acquisition functions that minimize classification error of the super-level set or control false discovery rates. While these approaches provide theoretical guarantees in certain settings, they are often formulated for sequential (single-point) selection and become computationally expensive or brittle in large-scale or batch settings. Moreover, many level-set methods focus on accurate boundary estimation rather than maximizing the number of recovered hits under a fixed budget, which is the primary objective in experimental screening contexts.

\section{Proof of Theorem~\ref{thm:poh_hits}}\label{app:proofs}
\label{app:proofs}

We prove Theorem~\ref{thm:poh_hits} under the batch feedback protocol of Section~\ref{sec:setting}.
Let $N:=Tb$ denote the total number of queried perturbations, and write each batch as
$B_t=\{g_{t,1},\dots,g_{t,b}\}$.
Recall that $p_t(g)=\mathbb{P}(f(g)>\tau\mid \mathcal{D}_t)$ and that PoH selects $B_t$ as the top-$b$ elements of $C_t$
ranked by $p_t(\cdot)$.

\paragraph{Step 0: Events.}
Let $\mathcal{E}_{\mathrm{conf}}$ be the event in Assumption~\ref{ass:gp_conf} (with confidence level $\delta/2$), i.e.,
\begin{equation}
\label{eq:conf_event_app}
\mathcal{E}_{\mathrm{conf}}
:=\left\{
|f(g)-\mu_t(g)| \le \beta_t \sigma_t(g),\ \forall g\in\mathcal{G},\ \forall t\le T
\right\},
\qquad
\mathbb{P}(\mathcal{E}_{\mathrm{conf}})\ge 1-\delta/2.
\end{equation}

\subsubsection*{A.1. Concentration of realized hits around posterior hit probabilities}

Define the hit indicators within a batch:
\[
X_{t,i} := \mathbf{1}\{ f(g_{t,i})>\tau\},
\qquad
H_T=\sum_{t=1}^T\sum_{i=1}^b X_{t,i}.
\]
Let $\mathcal{F}_t:=\sigma(\mathcal{D}_t)$ be the sigma-field generated by the information available at the start of round $t$.
Since all $b$ selections in batch $t$ are made based on $\mathcal{D}_t$, we have
\[
\mathbb{E}[X_{t,i}\mid \mathcal{F}_t] = p_t(g_{t,i}),
\qquad \forall t\in[T],\ i\in[b].
\]

\begin{lemma}[Hit-count concentration]
\label{lem:hit_concentration}
For any $\delta\in(0,1)$, with probability at least $1-\delta/2$,
\begin{equation}
\label{eq:hit_conc_app}
H_T
\;\ge\;
\sum_{t=1}^T\sum_{i=1}^b p_t(g_{t,i})
\;-\;
\sqrt{\tfrac{1}{2}N \log\tfrac{4}{\delta}}.
\end{equation}
\end{lemma}

\begin{proof}
Order the $N$ queried points in any fixed order that respects rounds, e.g.,
$n=(t-1)b+i$ corresponds to $(t,i)$, and define a martingale with bounded increments
$\sum_{m=1}^n (X_m-\mathbb{E}[X_m\mid \mathcal{F}_{t(m)}])$.
Azuma--Hoeffding yields \eqref{eq:hit_conc_app}.
\end{proof}

\subsubsection*{A.2. A lower bound on $\sum p_t(g)$ via PoH and the margin}

\begin{lemma}[One-sided CDF bound]
\label{lem:cdf_lip}
Let $p(z):=\Phi(z)$. For all $z\in\mathbb{R}$,
\begin{equation}
\label{eq:cdf_lip_app}
\Phi(z) \;\ge\; \frac{1}{2} - \frac{1}{\sqrt{2\pi}}\,(-z)_+.
\end{equation}
Equivalently, if $p_t(g)=\Phi((\mu_t(g)-\tau)/\sigma_t(g))$, then
\begin{equation}
\label{eq:pt_lb_app}
p_t(g)
\;\ge\;
\frac{1}{2}
-\frac{1}{\sqrt{2\pi}}
\left(\frac{\tau-\mu_t(g)}{\sigma_t(g)}\right)_+.
\end{equation}
\end{lemma}

\begin{proof}
Since $\Phi(0)=1/2$ and $\Phi'(z)=\phi(z)\le \phi(0)=1/\sqrt{2\pi}$, we have for $z\le 0$ that
$\Phi(z)\ge \Phi(0)-\phi(0)(-z)$, while for $z\ge 0$ the bound is trivial.
Substituting $z=(\mu_t(g)-\tau)/\sigma_t(g)$ yields \eqref{eq:pt_lb_app}.
\end{proof}

\begin{lemma}[Mistake implies large uncertainty]
\label{lem:mistake_large_sigma}
Fix a round $t$ and suppose there exists $g^\star\in C_t$ with $f(g^\star)\ge \tau+\epsilon$.
Assume $\mathcal{E}_{\mathrm{conf}}$ holds.
If PoH selects some $g\in B_t$ such that $f(g)\le \tau-\epsilon$, then
\begin{equation}
\label{eq:sigma_lower_app}
\sigma_t(g)\;\ge\;\frac{\epsilon}{2\beta_t}.
\end{equation}
\end{lemma}

\begin{proof}
On $\mathcal{E}_{\mathrm{conf}}$, for the remaining margin-hit $g^\star$ we have
$\mu_t(g^\star)\ge f(g^\star)-\beta_t\sigma_t(g^\star)\ge \tau+\epsilon-\beta_t\sigma_t(g^\star)$, hence
\[
\frac{\tau-\mu_t(g^\star)}{\sigma_t(g^\star)} \le \beta_t - \frac{\epsilon}{\sigma_t(g^\star)}.
\]
For the selected non-hit $g$ we have
$\mu_t(g)\le f(g)+\beta_t\sigma_t(g)\le \tau-\epsilon+\beta_t\sigma_t(g)$, so
\[
\frac{\tau-\mu_t(g)}{\sigma_t(g)} \ge \frac{\epsilon}{\sigma_t(g)}-\beta_t.
\]
PoH selects top-$b$ by $p_t(\cdot)$, equivalently it selects points with small
$(\tau-\mu_t(\cdot))/\sigma_t(\cdot)$, so for any selected $g\in B_t$ we must have
\[
\frac{\tau-\mu_t(g)}{\sigma_t(g)} \le \frac{\tau-\mu_t(g^\star)}{\sigma_t(g^\star)}.
\]
Combining the displays yields $\epsilon/\sigma_t(g)\le 2\beta_t$, proving \eqref{eq:sigma_lower_app}.
\end{proof}

\subsubsection*{A.3. Bounding the cumulative uncertainty under batch feedback}

\begin{lemma}[Sum of standard deviations under batch feedback]
\label{lem:batch_sigma_sum}
Let $\{B_t\}_{t=1}^T$ be any sequence of batches of size $b$ (possibly adaptive).
Then
\begin{equation}
\label{eq:batch_sigma_sum_app}
\sum_{t=1}^{T}\sum_{g\in B_t}\sigma_t(g)
\;\le\;
\sqrt{
\frac{2\gamma_{Tb}}{\log(1+\lambda^{-1})}
\left(
Tb \;+\; \frac{2b^2\,\gamma_T}{\log(1+\lambda^{-1})}
\right)}.
\end{equation}
\end{lemma}

\begin{proof}
This is Lemma~A.1 (Sum of standard deviations under batch feedback) from the appendix derivation.
\end{proof}

\subsubsection*{A.4. Completing the proof}

On $\mathcal{E}_{\mathrm{conf}}$, applying Lemma~\ref{lem:cdf_lip} to each queried point gives
\begin{equation}
\label{eq:sum_p_lb_app}
\sum_{t=1}^T\sum_{i=1}^b p_t(g_{t,i})
\;\ge\;
\frac{N}{2}
-\frac{1}{\sqrt{2\pi}}
\sum_{t=1}^T\sum_{g\in B_t}
\left(\frac{\tau-\mu_t(g)}{\sigma_t(g)}\right)_+.
\end{equation}
Using Lemma~\ref{lem:mistake_large_sigma} and the margin assumption, the cumulative positive-part term in
\eqref{eq:sum_p_lb_app} can be bounded by
\[
\mathcal{O}\!\left(\frac{\beta_{N}^{3}}{\epsilon^2}\right)\,
\sum_{t=1}^T\sum_{g\in B_t}\sigma_t(g).
\]
Substituting Lemma~\ref{lem:batch_sigma_sum} yields
\begin{equation}
\label{eq:sum_p_final_app}
\sum_{t=1}^T\sum_{i=1}^b p_t(g_{t,i})
\;\ge\;
\frac{N}{2}
-
\mathcal{O}\!\left(
\frac{\beta_{N}^{3}}{\epsilon^2}\,
\sqrt{
\frac{\gamma_{N}}{\log(1+\lambda^{-1})}
\left(
N \;+\; \frac{b^2\,\gamma_T}{\log(1+\lambda^{-1})}
\right)}
\right).
\end{equation}
Finally, combining \eqref{eq:sum_p_final_app} with Lemma~\ref{lem:hit_concentration} and taking a union bound with
\eqref{eq:conf_event_app} implies that with probability at least $1-\delta$,
\[
H_T
\;\ge\;
\frac{N}{2}
-
\sqrt{\tfrac{1}{2}N\log\tfrac{4}{\delta}}
-
\mathcal{O}\!\left(
\frac{\beta_{N}^{3}}{\epsilon^2}\,
\sqrt{
\frac{\gamma_{N}}{\log(1+\lambda^{-1})}
\left(
N \;+\; \frac{b^2\,\gamma_T}{\log(1+\lambda^{-1})}
\right)}
\right),
\]
which matches \eqref{eq:thm_poh_batch_final} (absorbing constants and the minor $\log(4/\delta)$ adjustment into the stated form).
This concludes the proof.
\qed

\section{Datasets}\label{app:datasets}
\subsection{Synthetic Datasets}\label{app:datasets_synthetic}

\subsubsection{Sine (1D).}
The simplest baseline dataset consists of a one-dimensional sinusoidal function with additive Gaussian noise. Given input $x \in [0, 1]$, the target is:
\begin{equation}
    y = \sin(2\pi x) + \epsilon, \quad \epsilon \sim \mathcal{N}(0, \sigma^2)
\end{equation}
where $\sigma = 0.05$. This dataset provides a minimal test case with a single global maximum at $x = 0.25$, enabling rapid iteration and debugging of acquisition functions before scaling to higher dimensions.

\subsubsection{Sine-2D.}
The two-dimensional extension uses a weighted combination of sinusoidal components with learnable structure. For input $\mathbf{x} = (x_1, x_2) \in [-\pi, \pi]^2$, the target is:
\begin{equation}
    y = \frac{1}{K} \sum_{k=1}^{K} \sin\left(\mathbf{w}_k^\top \mathbf{x} + \phi_k\right)
\end{equation}
where $\mathbf{w}_k \in \mathbb{R}^2$ are weight vectors, $\phi_k$ are phase shifts, and $K=2$ components are used. The default weight matrix is:
\begin{equation}
    \mathbf{W} = \begin{pmatrix} 0.25 & -1/\pi \\ 0.1 & 0.02 \end{pmatrix} + \boldsymbol{\epsilon}_W, \quad \boldsymbol{\epsilon}_W \sim \mathcal{N}(0, 0.05^2 \mathbf{I})
\end{equation}
with seed-dependent jitter $\boldsymbol{\epsilon}_W$ and phases $\phi_k \sim \mathcal{U}(-\pi, \pi)$. This creates distinct but comparable landscapes across random seeds, testing robustness to function variability.

\subsubsection{Branin-Hoo (2D).}
The Branin-Hoo function is a classical optimization benchmark with three global minima, commonly used to evaluate global optimization algorithms. The standard formulation is:
\begin{equation}
    f(x_1, x_2) = a\left(x_2 - b x_1^2 + c x_1 - r\right)^2 + s(1 - t)\cos(x_1) + s
\end{equation}
with canonical parameters $a = 1$, $b = 5.1/(4\pi^2)$, $c = 5/\pi$, $r = 6$, $s = 10$, $t = 1/(8\pi)$. The original domain $x_1 \in [-5, 10]$, $x_2 \in [0, 15]$ is rescaled to the unit square $[0, 1]^2$ via:
\begin{equation}
    \tilde{x}_1 = 15 x_1 - 5, \quad \tilde{x}_2 = 15 x_2
\end{equation}
We invert and normalize the function for maximization:
\begin{equation}
    y = \frac{-f(\tilde{x}_1, \tilde{x}_2) - y_{\min}}{y_{\max} - y_{\min}} + \epsilon, \quad \epsilon \sim \mathcal{N}(0, 0.02^2)
\end{equation}
The three global maxima (originally minima) at $(\pi, 2.275)$, $(-\pi, 12.275)$, and $(9.42, 2.475)$ test whether acquisition functions can discover multiple optima rather than converging prematurely to a single peak.

The following datasets simulate the structure of gene perturbation experiments, where the goal is to identify genes whose perturbation produces a desired phenotype.

\subsubsection{Perturbation-Pathways (4D).}
This dataset introduces \emph{pathway structure}---a core organizing principle of molecular biology where genes with related functions cluster into functional modules. The feature vector $\mathbf{x} \in \mathbb{R}^4$ consists of:
\begin{equation}
    \mathbf{x} = [\underbrace{\mathbf{g}}_{\text{gene} \in \mathbb{R}^2}, \underbrace{\mathbf{z}}_{\text{context} \in \mathbb{R}^2}]
\end{equation}

Genes are organized into $K = 4$ pathways, each centered at a distinct location in the 2D gene-embedding space:
\begin{align}
    & \boldsymbol{\mu}_1 = (0.2, 0.2), 
    \boldsymbol{\mu}_2 = (0.8, 0.2), \\&
    \boldsymbol{\mu}_3 = (0.2, 0.8), 
    \boldsymbol{\mu}_4 = (0.8, 0.8)
\end{align}
Each of the $n_{\text{genes}}$ genes is assigned to a pathway, and its embedding is sampled around the pathway center:
\begin{equation}
    \mathbf{g}_i \sim \mathcal{N}(\boldsymbol{\mu}_{k(i)}, \sigma_g^2 \mathbf{I}), \quad \sigma_g = 0.12
\end{equation}
where $k(i)$ denotes the pathway assignment of gene $i$. This creates spatial clusters in embedding space, reflecting the biological reality that genes in the same pathway have similar functional annotations.

The phenotype combines pathway-level and gene-specific effects:
\begin{align}
    y &= \underbrace{A_{k} \exp\left(-\frac{\|\mathbf{z} - \boldsymbol{\mu}_k\|^2}{2\sigma_k^2}\right)}_{\text{pathway activation}} + \underbrace{\alpha \sin(4\pi g_1) \cos(4\pi g_2)}_{\text{gene-specific modulation}} \nonumber \\
    &\quad + \underbrace{0.3 \sin(2\pi(g_1 + z_1)) \cos(2\pi(g_2 + z_2))}_{\text{gene} \times \text{context interaction}} + \epsilon
\end{align}
where $A_k \sim \mathcal{U}(0.8, 1.5) \cdot 2.5$ and $\sigma_k \sim \mathcal{U}(0.15, 0.25)$ are pathway-specific amplitude and width parameters, $\alpha = 0.4$ scales the gene-specific effect, and $\epsilon \sim \mathcal{N}(0, 0.08^2)$.

The key insight is the \emph{pathway activation term}: when the context $\mathbf{z}$ is close to pathway center $\boldsymbol{\mu}_k$, perturbing \emph{any} gene in that pathway produces a strong phenotypic effect. This models the biological phenomenon where:
\begin{itemize}[leftmargin=*,nosep]
    \item Pathway activity is context-dependent (e.g., a signaling pathway may only be active in certain cell states).
    \item Genes within an active pathway are functionally redundant for the phenotype.
    \item The Gaussian activation creates smooth spatial structure that is learnable by GP-like models with RBF kernels.
\end{itemize}

\subsubsection{Perturbation-SEM (6D).}
This dataset models single-cell perturbation experiments with explicit causal structure. Each sample represents a perturbation characterized by three latent factors: \emph{gene identity}, \emph{cell state}, and \emph{environmental condition}. The feature vector $\mathbf{x} \in \mathbb{R}^6$ is partitioned as:
\begin{equation}
    \mathbf{x} = [\underbrace{\mathbf{g}}_{\text{gene} \in \mathbb{R}^2}, \underbrace{\mathbf{c}}_{\text{cell} \in \mathbb{R}^2}, \underbrace{\mathbf{e}}_{\text{env} \in \mathbb{R}^2}]
\end{equation}
where $\mathbf{g}$ is a 2D gene embedding drawn from a discrete set of $n_{\text{genes}}$ pre-defined gene positions, while $\mathbf{c}$ and $\mathbf{e}$ represent continuous cell-state and environment coordinates sampled uniformly from $[0, 1]^2$.

The phenotype $y$ is generated via an additive structural equation model with interaction terms:
\begin{align}
    y &= \underbrace{1.5 \sin(2\pi g_1) \cos(2\pi g_2)}_{\text{gene main effect}} +
    \\& \underbrace{0.8 \sin(\pi c_1) + 0.4 \cos(\pi c_2)}_{\text{cell-state effect}} \nonumber \\
    &\quad + \underbrace{0.6 \sin(\pi e_1) \cos(\pi e_2)}_{\text{environment effect}} + \\& \underbrace{0.5 \sin(\pi(g_1 + c_1)) \cos(\pi(g_2 + c_2))}_{\text{gene} \times \text{cell interaction}} + \epsilon
\end{align}
where $\epsilon \sim \mathcal{N}(0, 0.08^2)$.

This formulation captures the biological intuition that gene identity determines a baseline perturbation effect via the smooth gene-embedding function.
Moreover, Cell state modulates the phenotypic readout (e.g., different cell types respond differently), and environmental conditions provide additional context (e.g., drug treatment, growth factors). The gene$\times$cell interaction term captures context-dependent gene effects, a key challenge in biological screens where the same perturbation can have opposite effects in different cellular contexts.

\subsection{Dataset Complexity}\label{app:dataset_complexity}
\paragraph{Dataset Complexity Metrics.}
We design the following metrics to characterize the optimization landscapes underlying these 10 datasets.
\begin{itemize}[leftmargin=*,nosep]
    \item \textbf{Local Smoothness} ($\mathcal{S}$): The average absolute difference in phenotype values between $k$-nearest neighbors in feature space. Higher values indicate more rugged phenotype landscapes that are harder for surrogate models to approximate.
    
    \item \textbf{Hit Cluster Count} ($n_{\text{clusters}}$): The number of spatially distinct hit clusters identified by DBSCAN in the normalized feature space. Multiple clusters indicate fragmented hit regions requiring exploration.
    
    \item \textbf{Effective Dimensionality} ($d_{\text{eff}}$): The number of principal components required to explain 95\% of variance in the feature space. High values indicate intrinsically high-dimensional data.
    
    \item \textbf{Feature--Phenotype Correlation} ($\rho_{\max}$): The maximum absolute Spearman correlation between any single feature and the phenotype. Low values indicate complex, nonlinear relationships.
    
    \item \textbf{Distance--Phenotype Correlation} ($\rho_{d,y}$): Spearman correlation between pairwise distances and phenotype differences. Measures how well proximity in feature space predicts phenotypic similarity.
\end{itemize}

\subsection{Real Datasets.}
To bridge the gap between our synthetic ablation studies and real-world experimental design, we conduct a comprehensive complexity analysis of five GeneDisco benchmark datasets~\citep{mehrjou2021genedisco}. These datasets represent genome-wide CRISPR knockout screens targeting diverse biological phenotypes: T-cell activation (IFN-$\gamma$ and IL-2 production), cancer cell killing (NK cells), neurodegeneration (tau protein), and viral infection (SARS-CoV-2).
Table~\ref{tab:real_data_overview_complexity} summarizes the basic characteristics of each dataset.

\begin{table*}[htbp]
\centering
\caption{Overview and complexity metrics for all GeneDisco datasets using Achilles features.
Sample sizes vary due to gene overlap between screen results and feature databases.
Smoothness ($\mathcal{S}$) is in natural target units.
$d_{\text{eff}}$ is effective dimensionality at 95\% variance.}
\label{tab:real_data_overview_complexity}
\scriptsize
\setlength{\tabcolsep}{5pt}
\begin{tabular}{lrrrrrrrr}
\toprule
\textbf{Dataset} &
\textbf{Genes} & \textbf{Feats.} & \textbf{Hits (10\%)} &
\textbf{$\mathcal{S}$} & \textbf{Clusters} & \textbf{$d_{\text{eff}}$} &
\textbf{$\rho_{\max}$} & \textbf{$\rho_{d,y}$} \\
\midrule

Schmidt IFN-$\gamma$ & 17{,}470 & 808 & 1{,}747 & 0.215 & 2 & 51 & 0.038 & 0.022 \\
Schmidt IL-2         & 17{,}470 & 808 & 1{,}747 & 0.207 & 7 & 51 & 0.113 & 0.106 \\
Zhuang NK            & 17{,}533 & 808 & 1{,}754 & 1.741 & 6 & 51 & 0.060 & 0.006 \\
Sanchez Tau          & 17{,}181 & 808 & 1{,}719 & 0.625 & 3 & 51 & 0.068 & 0.054 \\
Zhu SARS-CoV-2       & 16{,}751 & 808 & 1{,}676 & 0.913 & 1 & 51 & 0.154 & 0.126 \\

\bottomrule
\end{tabular}
\end{table*}

\paragraph{Key Observations.}

\begin{enumerate}
    \item \textbf{Schmidt IL-2 has the smoothest landscape}: With local smoothness $\mathcal{S} = 0.22$, this dataset has the most learnable phenotype function. Neighboring genes in feature space tend to have similar IL-2 production effects, making GP-based surrogates effective.
    
    \item \textbf{Zhuang NK is the most challenging}: With $\mathcal{S} = 1.76$, the NK cell killing phenotype shows high local variability. Genes with similar embeddings can have drastically different effects on NK cytotoxicity.
    
    \item \textbf{SARS-CoV-2 has the strongest feature--phenotype correlations}: The maximum single-feature correlation ($\rho_{\max} = 0.134$) suggests some features are individually predictive of viral susceptibility, which may aid acquisition functions.
    
    \item \textbf{All datasets have moderate hit clustering}: Between 10--17 DBSCAN clusters indicates that hits are distributed across multiple regions of the feature space, requiring exploration to achieve comprehensive coverage.
\end{enumerate}

\paragraph{Practical Implications}

Based on the complexity analysis, we offer dataset-specific recommendations:

\paragraph{Schmidt IL-2/IFN-$\gamma$ (T-cell activation).}
These datasets have the smoothest landscapes ($\mathcal{S} \approx 0.22$), making them most amenable to model-based acquisition. The \texttt{probability\_hit} method, which leverages GP uncertainty, should perform well here.

\paragraph{Zhu SARS-CoV-2 (viral infection).}
Despite moderate smoothness, this dataset has the strongest single-feature predictors ($\rho_{\max} = 0.15$). Feature selection or regularization may improve surrogate model quality.

\paragraph{Zhuang NK (cancer killing).}
With the highest smoothness ($\mathcal{S} = 1.77$), this dataset poses the greatest challenge for model-based methods. Increased exploration or larger batch sizes may be necessary to overcome the rugged landscape.

\section{Acquisition Function: Thompson Hit}\label{app:thompson_hit}
\begin{algorithm}[h
]
\caption{Thompson-Hit}
\label{alg:thompsonhit}
\begin{algorithmic}[1]
\REQUIRE Candidate set $\mathcal{G}$, threshold $\tau$, batch size $b$, rounds $T$, posterior over $f$
\STATE Initialize $\mathcal{D}_1 \gets \emptyset$
\FOR{$t = 1$ to $T$}
    \STATE Fit/condition posterior on $\mathcal{D}_t$
    \STATE $C_t \gets \mathcal{G} \setminus \bigcup_{s<t} B_s$
    \STATE Sample $\tilde f_t \sim p(f \mid \mathcal{D}_t)$
    \STATE $\tilde H_t \gets \{ g \in C_t : \tilde f_t(g) > \tau \}$
    \IF{$|\tilde H_t| \ge b$}
        \STATE $B_t \gets$ UniformRandomSubset$(\tilde H_t, b)$
    \ELSE
        \STATE $B_t \gets \tilde H_t$
        \STATE $R_t \gets C_t \setminus \tilde H_t$
        \STATE Sort $R_t$ in descending order of $\tilde f_t(g)$
        \STATE Add the first $(b-|\tilde H_t|)$ elements of $R_t$ to $B_t$
    \ENDIF
    \STATE Evaluate $\{y_g : g \in B_t\}$
    \STATE $\mathcal{D}_{t+1} \gets \mathcal{D}_t \cup \{(g,y_g): g \in B_t\}$
\ENDFOR
\end{algorithmic}
\end{algorithm}

\newpage

\section{Additional Results}\label{app:additional_results}
\subsection{Synthetic Datasets}

\subsubsection{DiscoBax}
\label{app:discobax}
\begin{figure}[h]
    \centering
    \includegraphics[width=0.7\columnwidth]{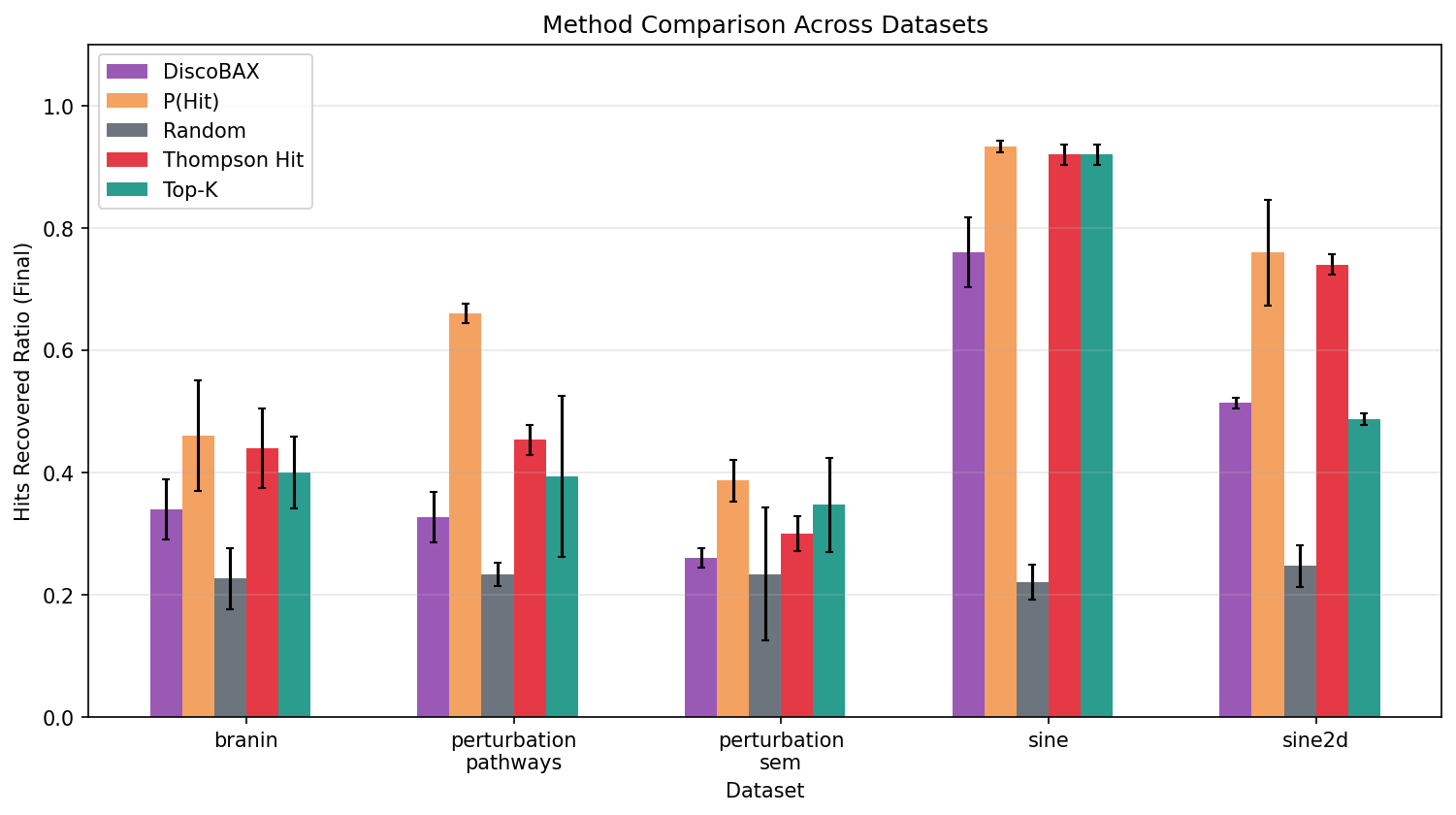}
    \caption{Method comparison with DiscoBax on 10 cycles, batches of 5 points, and threshold of 10\% (over 3 random seeds). DiscoBax consistently performs better than Random, but still significantly worse than the other methods.}
    \label{app:fig:discobax_method_comparison_bar}
\end{figure}

\begin{figure}[h]
    \centering
    \includegraphics[width=0.7\columnwidth]{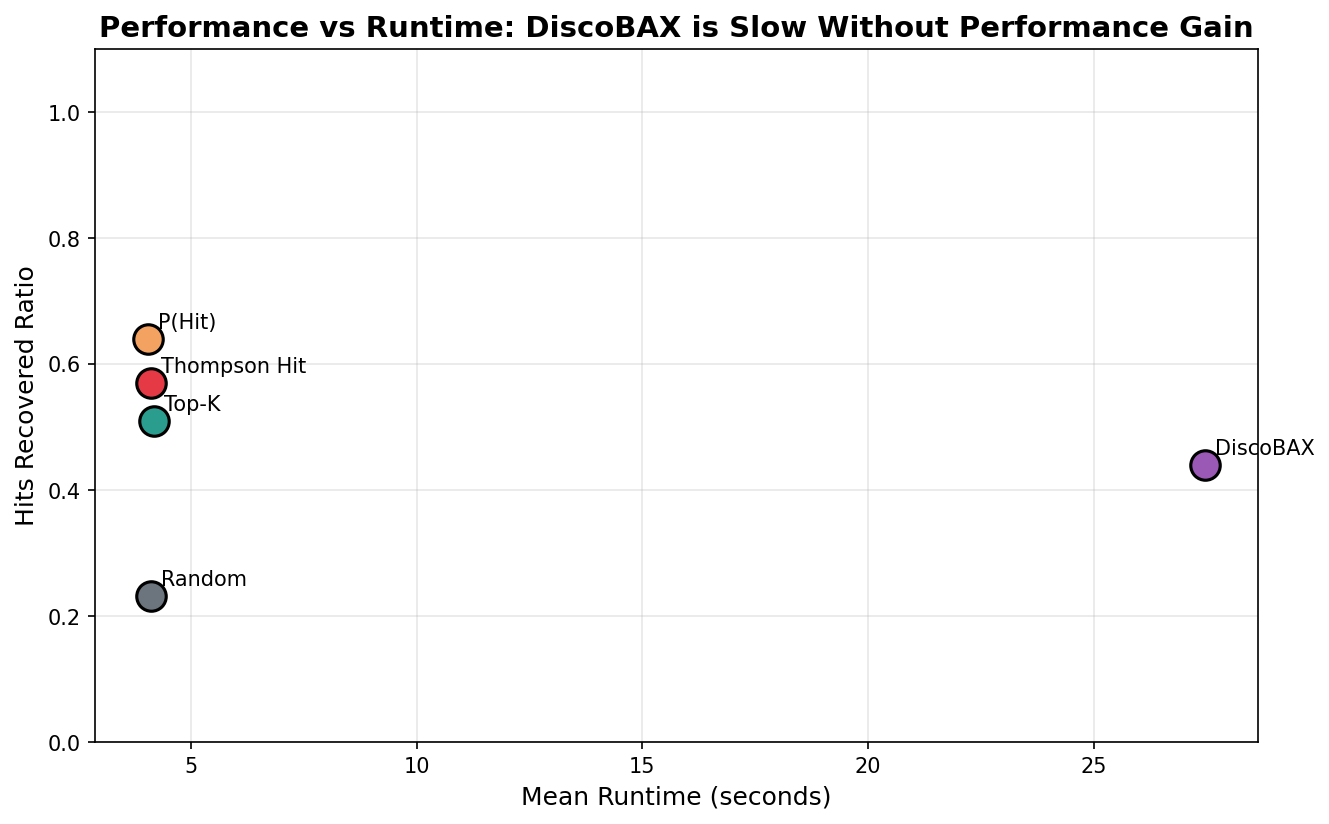}
    \caption{Hits Recovered Ratio vs. Mean Runtime on 10 cycles with batches of 5 (over 3 random seeds). DiscoBax's significantly longer runtimes, together with a sub-par performance with respect to the other methods, place it in the far right side of this plot, showcasing its general inefficiency.}
    \label{app:fig:performance_vs_runtime}
\end{figure}

\newpage

\subsubsection{Ablation: Batch-Size}

\begin{figure}[h]
    \centering
    \includegraphics[width=\columnwidth]{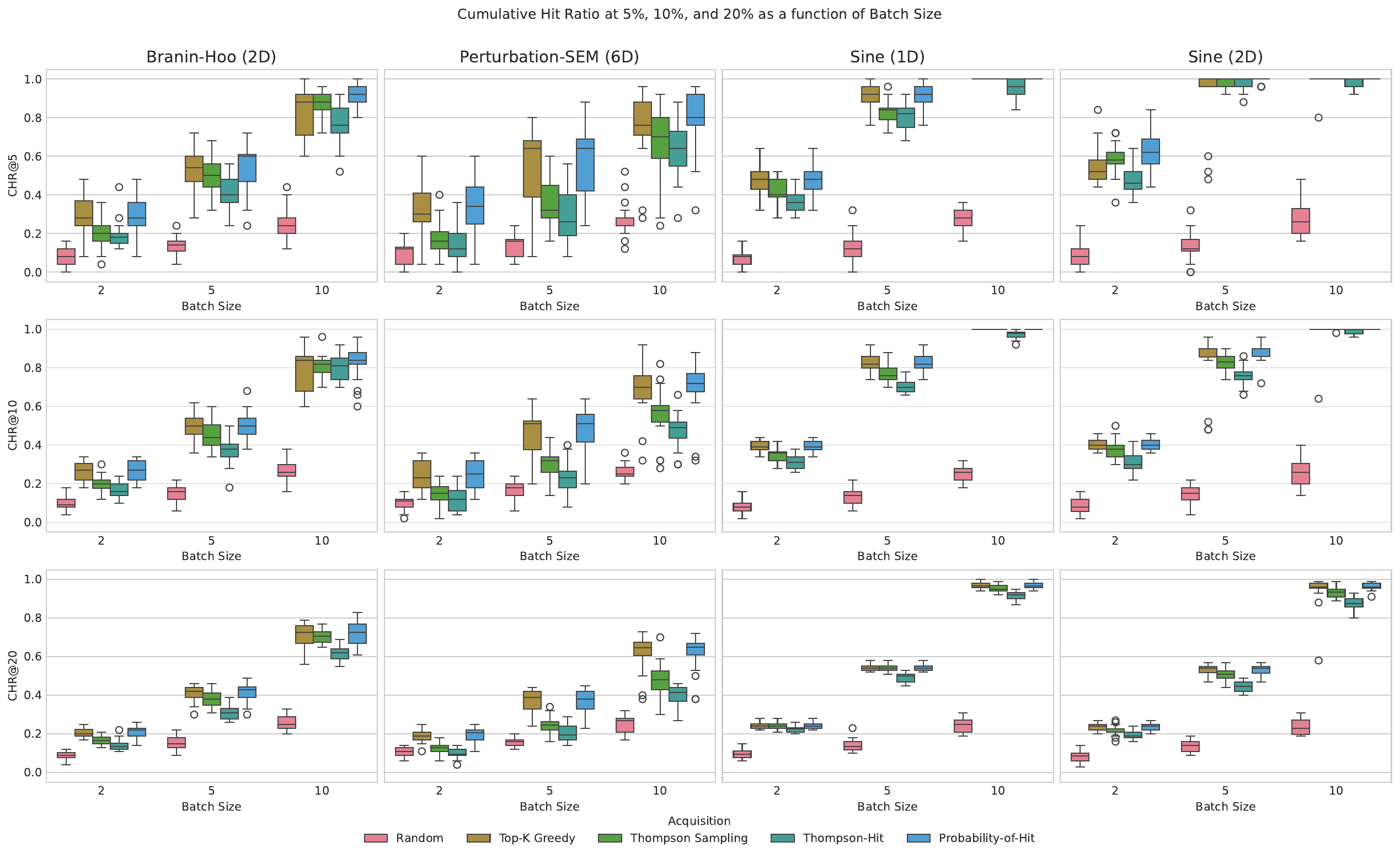}
    \caption{Impact of batch size on hit recovery for different thresholds across datasets for all acquisition functions.}
    \label{app:fig:batch_size_full_ablation}
\end{figure}

\newpage

\subsubsection{Exploration vs. Exploitation Trade-off}

\begin{figure}[h]
    \centering
    \includegraphics[width=0.9\columnwidth]{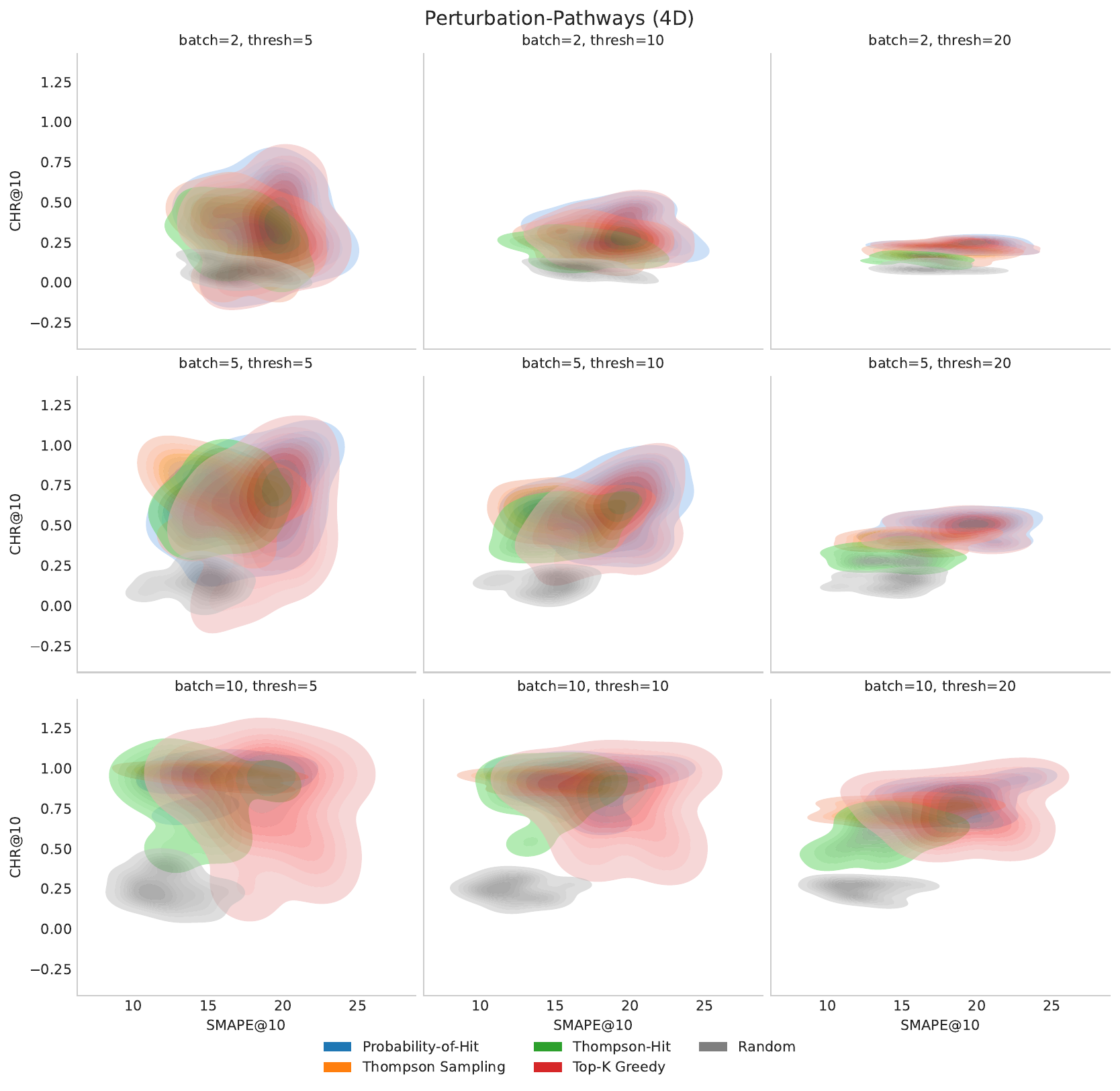}
    \caption{Trade-off between CHR@10 and SMAPE@10 on Perturbation-Pathways (4D) across acquisition functions, batch sizes, and thresholds.}
    \label{app:fig:pert_path_tradeoff}
\end{figure}

\newpage

\subsubsection{Learning Curves}

\begin{figure}[h]
    \centering
    \includegraphics[width=0.9\columnwidth]{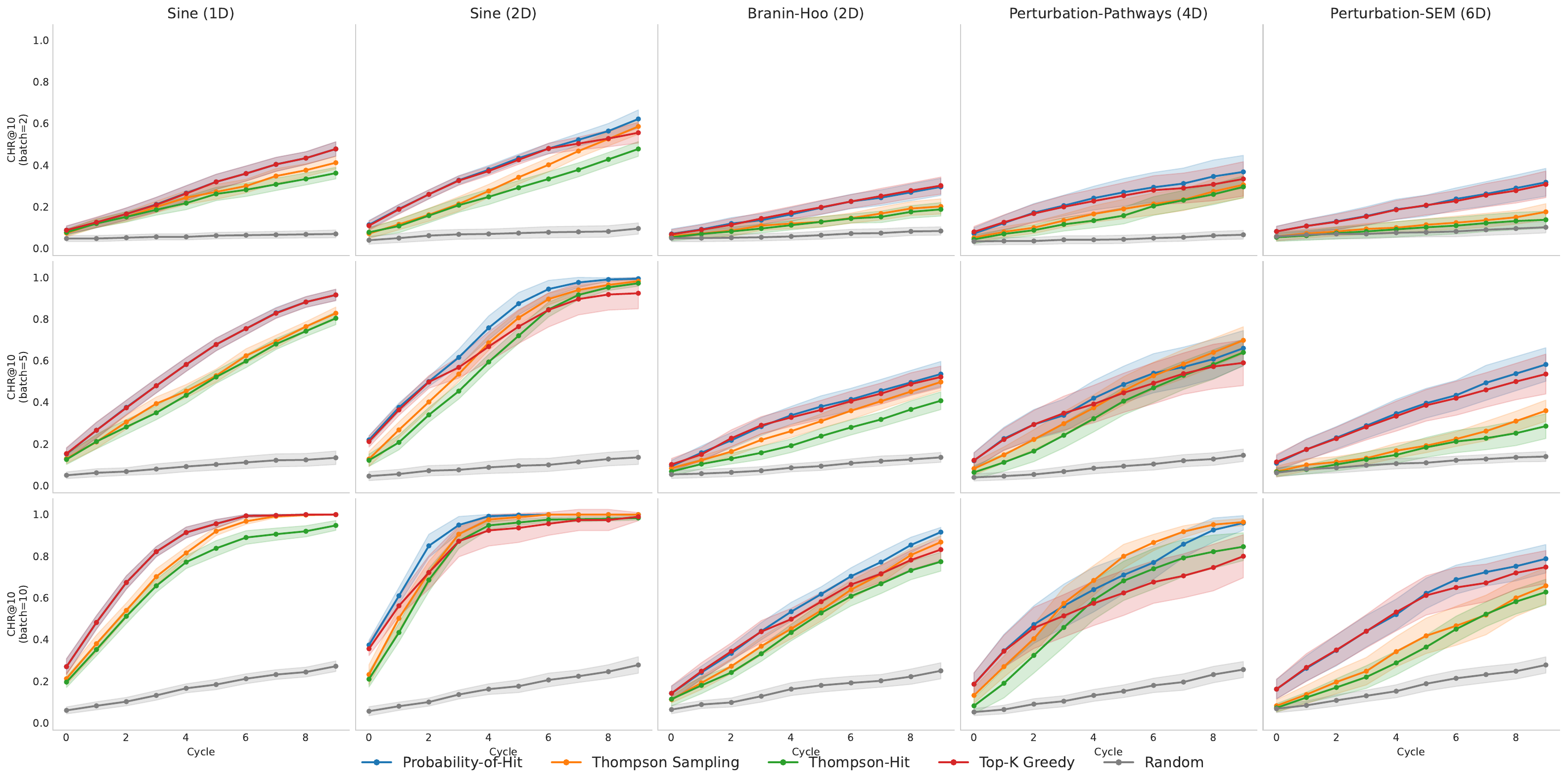}
    \caption{Learning curves for different acquisition functions across datasets and batch sizes for $\tau=5\%$.}
    \label{app:fig:learning_curves_tau5}
\end{figure}

\begin{figure}[h]
    \centering
    \includegraphics[width=0.9\columnwidth]{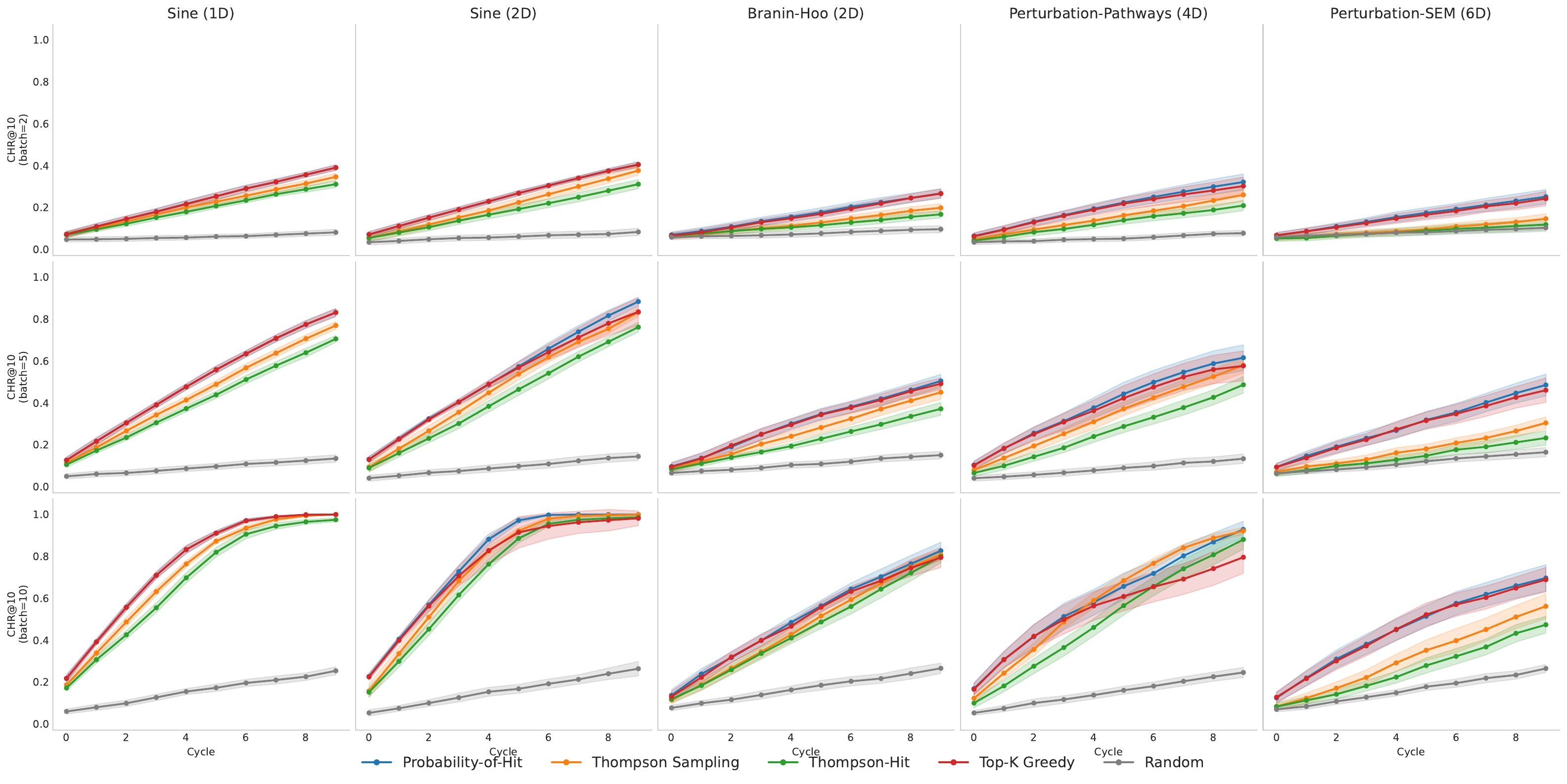}
    \caption{Learning curves for different acquisition functions across datasets and batch sizes for $\tau=20\%$.}
    \label{app:fig:learning_curves_tau10}
\end{figure}

\begin{figure}[h]
    \centering
    \includegraphics[width=0.9\columnwidth]{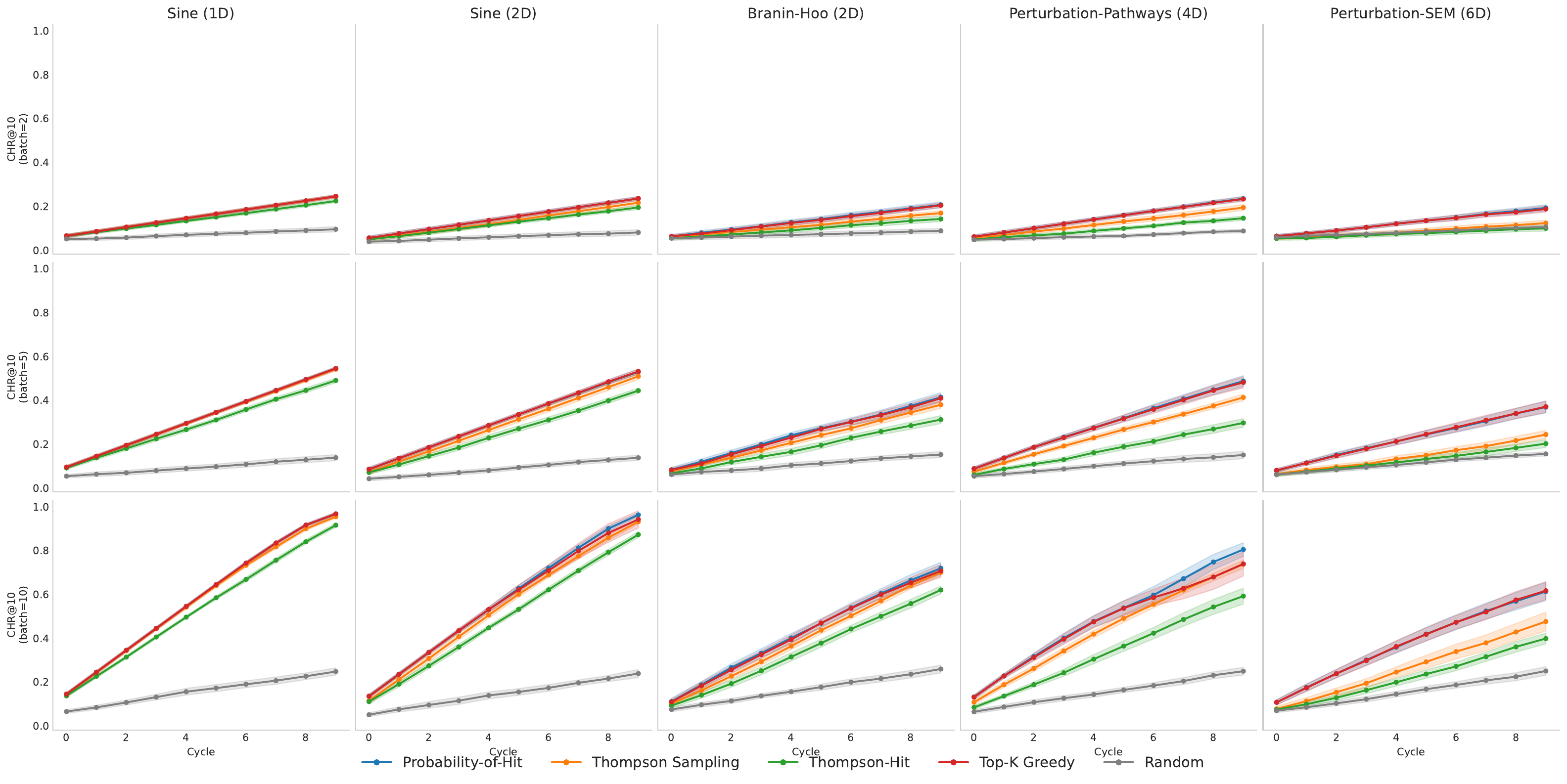}
    \caption{Learning curves for different acquisition functions across datasets and batch sizes for $\tau=20\%$.}
    \label{app:fig:learning_curves_tau20}
\end{figure}

\newpage

\subsection{Real-World Datasets}

The Schmidt IL-2 dataset with Achilles features shows the strongest Probability-Hit advantage: \textbf{+58.8 hits (+11.4\%)}. This dataset has the lowest local smoothness ($\mathcal{S} = 0.21$) in our complexity analysis, suggesting Probability-Hit excels when the phenotype landscape is learnable.

\begin{table*}[htbp]
\footnotesize
\centering
\caption{Complete hit recovery results at 5 cycles. Values are mean $\pm$ std across 5 seeds. Best method per row is \textbf{bolded}.}
\label{tab:full_results_5cycles}
\scriptsize
\resizebox{\textwidth}{!}{%
\begin{tabular}{llccccc}
\toprule
\textbf{Dataset} & \textbf{Feature} & \textbf{Probability-Hit} & \textbf{Thompson} & \textbf{Thom-Hit} & \textbf{Top-K} & \textbf{Random} \\
\midrule
Schmidt IL-2 & Achilles & \textbf{336.6$\pm$62} & 324.4$\pm$66 & 324.6$\pm$76 & 304.6$\pm$80 & 120.4$\pm$12 \\
Sanchez Tau  & Achilles & \textbf{212.6$\pm$18} & 212.6$\pm$11 & 207.0$\pm$14 & 209.8$\pm$20 & 116.4$\pm$7 \\
Schmidt IFN-$\gamma$ & Achilles & \textbf{195.8$\pm$36} & 187.2$\pm$28 & 176.6$\pm$38 & 179.0$\pm$27 & 114.4$\pm$12 \\
Zhu SARS-CoV-2 & Achilles & 137.6$\pm$21 & 135.4$\pm$12 & 135.0$\pm$11 & \textbf{141.4$\pm$20} & 120.8$\pm$6 \\
Zhuang NK    & Achilles & 128.8$\pm$10 & 126.8$\pm$10 & 127.0$\pm$10 & \textbf{129.4$\pm$12} & 122.6$\pm$20 \\
\bottomrule
\end{tabular}
}
\end{table*}

\begin{table}[htbp]
\centering
\caption{SMAPE (\%) on unseen data by acquisition method (final cycle)}
\label{tab:smape_achilles}
\begin{tabular}{lccccc}
\toprule
Dataset & Probability Hit & Random & Thompson Hit & Thompson Sampling & Top-$k$ \\
\midrule
Sanchez 2021 (Tau)    & $162.84 \pm 44.85$ & $149.34 \pm 43.42$ & $162.26 \pm 51.21$ & $\mathbf{130.88 \pm 63.35}$ & $158.26 \pm 45.67$ \\
Schmidt 2021 (IFN-$\gamma$) & $147.38 \pm 9.11$  & $145.42 \pm 10.47$ & $145.02 \pm 9.04$  & $140.95 \pm 6.63$  & $\mathbf{140.51 \pm 6.09}$ \\
Schmidt 2021 (IL-2)   & $144.91 \pm 5.81$  & $139.12 \pm 7.39$  & $141.14 \pm 7.26$  & $\mathbf{138.83 \pm 7.68}$  & $146.51 \pm 4.56$ \\
Zhu 2021 (SARS-CoV-2) & $188.59 \pm 13.91$ & $\mathbf{152.93 \pm 51.27}$ & $179.50 \pm 41.29$ & $188.50 \pm 15.42$ & $177.57 \pm 39.38$ \\
Zhuang 2019 (NK)      & $165.56 \pm 14.70$ & $170.51 \pm 9.46$  & $\mathbf{161.47 \pm 8.60}$  & $162.25 \pm 7.88$  & $169.29 \pm 10.88$ \\
\bottomrule
\end{tabular}
\end{table}

\paragraph{Practical Recommendations.}
Based on our analysis, we offer the following recommendations for practitioners:

\begin{enumerate}
    \item \textbf{Prioritize batch size over threshold refinement}: Increasing the number of genes tested per cycle provides substantially larger performance gains than adjusting the hit threshold. If resources permit, larger batches should be preferred.
    
    \item \textbf{Use hit-aware acquisition functions}: Methods like \texttt{probability\_hit} that explicitly model hit probability outperform value-based methods, especially in high-dimensional settings with dispersed hits.
    
    \item \textbf{Expect diminishing returns with dimensionality}: As feature dimensionality increases (e.g., using richer gene embeddings), surrogate model quality degrades significantly. Consider dimensionality reduction or feature selection for very high-dimensional representations.
    
    \item \textbf{Account for pathway structure}: When hits cluster by biological pathway, methods that can discover multiple disconnected regions (through exploration) may be valuable for achieving comprehensive pathway coverage.
\end{enumerate}

\paragraph{Biological Interpretation.}
The finding that hit cluster count increases with threshold (Table~\ref{tab:hit_clusters_threshold}) has an interesting biological interpretation. In genetic screens, the most extreme phenotypes (top 5\%) often arise from perturbations in a single critical pathway, while milder phenotypes (top 20\%) may involve multiple pathways with redundant or compensatory functions. Our analysis suggests that hit-aware methods are particularly valuable when seeking to discover these diverse biological mechanisms rather than just the strongest individual effects.



\end{document}